\documentclass[10pt,a4paper]{scrartcl}

\usepackage{etoolbox}
\newtoggle{release}
\newtoggle{arxiv}
\toggletrue{release}
\toggletrue{arxiv}
\usepackage[utf8]{inputenc}
\usepackage{times}
\usepackage[T1]{fontenc}
\usepackage{xcolor}
\usepackage{amsmath, amssymb, amsthm}
\usepackage{enumerate}
\usepackage{wrapfig}
\usepackage{hyperref}
\usepackage{url}
\usepackage{pgf}
\usepackage{tikz}
\usepackage{tikzscale}
\usepackage{bm}
\usepackage{pgfplots}
\usepackage{algorithm}
\usepackage{algorithmicx}
\usepackage{algpseudocode}
\usepackage{multicol}
\usepackage{bbm}

\newcommand{\mysec}[1]{Section~\ref{sec:#1}}
\newcommand{\mysecan}[1]{Section~\ref{an:sec:#1}}

\usepackage{pgfplotstable}
\pgfplotsset{compat=newest}
\usepgfplotslibrary{external}

\ifx\note\undefined
\nottoggle{release}{
\newcommand{\note}[1]{{\textbf{\color{red}#1}}}
\newcommand{\noteal}[1]{{\textbf{\color{blue}#1}}}
}{
\newcommand{\note}[1]{}
\newcommand{\noteal}[1]{}
}
\fi

\usetikzlibrary{calc,positioning}
\usepackage{subcaption}
\usepackage{caption}

\usepackage{graphicx}
\usepackage{epstopdf}
\DeclareMathOperator*{\argmin}{arg\,min}
\providecommand{\abs}[1]{\left|#1\right|}
\providecommand{\norm}[1]{\left\|#1\right\|}

\providecommand{\esp}[1]{\mathbb{E}\left[#1\right]}
\providecommand{\variance}[1]{\mathbb{V}\left[#1\right]}

\providecommand{\diag}[1]{\mathrm{diag}\left(#1\right)}
\providecommand{\Diag}[1]{\mathrm{Diag}\left(#1\right)}
\providecommand{\prob}[1]{\mathbb{P}\left\{#1\right\}}

\providecommand{\reel}{\mathbb{R}}
\providecommand{\Suppo}[1]{\mathcal{S}\left(#1\right)}

\providecommand{\pmin}{p_{\mathrm{min}}}
\providecommand{\pmax}{p_{\mathrm{max}}}

\DeclareMathOperator{\Id}{Id}
\numberwithin{equation}{section}

\providecommand{\gmax}{\gamma_{\mathrm{max}}}

\ifx\lemma\undefined
\newtheorem{thm}{Theorem}
\newtheorem{corollary}{Corollary}
\newtheorem{lemma}{Lemma}
\newtheorem{assumption}{Assumption}
\fi

\providecommand{\gab}[1]{g_{\mathrm{ab}, #1}}
\providecommand{\gmb}[1]{g_{\mathrm{mb}, #1}}
\providecommand{\grc}[1]{g_{\mathrm{C}, #1}}

\usepackage[numbers]{natbib}

\algblock{ParFor}{EndParFor}
\algnewcommand\algorithmicparfor{\textbf{parallel for}}
\algnewcommand\algorithmicpardo{\textbf{do}}
\algnewcommand\algorithmicendparfor{\textbf{end\ parallel for}}
\algrenewtext{ParFor}[1]{\algorithmicparfor\ #1\ \algorithmicpardo}
\algrenewtext{EndParFor}{\algorithmicendparfor}

\usepackage{geometry}
\begin{document}

\title{AdaBatch: Efficient Gradient  Aggregation Rules for Sequential and Parallel Stochastic Gradient Methods}
\author{\textbf{Alexandre Défossez} \\
Facebook AI Research\\
Paris, France\\
\texttt{defossez@fb.com} \\
 \and  \textbf{Francis Bach} \\
D\'epartement d'Informatique\\
\'Ecole Normale Sup\'erieure \\
 Paris, France \\
 \texttt{francis.bach@inria.fr} \\
}

\maketitle
\begin{abstract}
We study a new aggregation operator for gradients coming from a mini-batch for \emph{stochastic gradient} (SG) methods that allows a significant speed-up in the case of sparse optimization problems. We call this method AdaBatch and it only requires a few lines of code change compared to   regular mini-batch SGD algorithms.
We provide a theoretical insight to understand how this new class of algorithms is performing and show that it is equivalent to an implicit per-coordinate rescaling of the gradients, similarly to what Adagrad methods can do.
In theory and in practice, this new aggregation allows to keep the same sample efficiency of SG methods while increasing the batch size. Experimentally, we also show that in the case of smooth convex optimization, our procedure can even obtain a better loss when increasing the batch size for a fixed number of samples.
We then apply this new algorithm to obtain a parallelizable stochastic gradient method that is synchronous but allows speed-up on par with Hogwild! methods as convergence does not deteriorate with the increase of the batch size. The same approach can be used to make mini-batch provably efficient for variance-reduced SG methods such as SVRG.

\end{abstract}

\section{Introduction}

We  consider large-scale supervised learning with sparse features, such as logistic regression, least-mean-square or support vector machines, with a very large dimension as well as a very large number
of training samples with many zero elements, or even an infinite stream.
A typical example of such use of machine learning is given by Ads click prediction where many sparse features can be used to improve prediction on a problem with a massive online usage. For such problems, \emph{stochastic gradient} (SG) methods have been used successfully \cite{NIPS2007_3323,bach11,ads_from_trenches}.

Sparse optimization requires the usage of CPUs and unlike other domains in machine learning, it did not benefit much from the ever increasing parallelism accessible in GPUs or dedicated hardware. The frequency of CPUs has been stagnating and we can no longer rely on the increase of CPU sequential computational power for SG methods to scale with the increase of data \cite{multicore}.
New CPUs now rely on multi-core and sometimes multi-socket  design to offer more power to its users. As a consequence, many attemps have been made at parallelizing and distributing SG methods \cite{parallel_sgd_average,slow_learners_are_fast,Hogwild,passcode}. Those approaches can be classified in two types: (a) synchronous methods, that seek a speed-up while staying logically equivalent to a sequential algorithm,
    (b) asynchronous methods, which allow some differences and approximations from the sequential algorithms, such as allowing delays in gradient updates, dropping overlapping updates from different workers or allowing inconsistent reads from the model parameters.
 The latter such as Hogwild!~\cite{Hogwild} have been more successful as the synchronization overhead between workers from synchronous methods made them impractical.

Such methods however do not lead to a complete provability of convergence for step-sizes used in practice, as most proof methods require some approximation \cite{Hogwild,hogwild_mania}.
Proving convergence for such methods is not as straightforward as there is not anymore
a clear sequence of iterates that actually exist in memory and
conflicting writes to memory or inconsistent reads can occur. When increasing the number of
workers it is also likely to increase how stale a gradient update is when being processed.

Synchronous approaches rely mostly on the usage of mini-batches in order to parallelize the workload~\cite{online_minibatch}. Increasing the size of the mini-batches will lead to a reduction of the variance of the gradients and the overall estimator. However
for the same number of samples we will be doing $B$ times less iterations
where $B$ is the size of the mini-batch. In practice one has to increase the learning
rate (i.e., the step size) in order to compensate and achieve the same final accuracy as without mini-batches; however increasing the step size can lead to divergence and is sometimes impossible~\cite{minibatch}. The decrease in sample efficiency (i.e., a worse performance for a given number of processed training samples) is especially visible early
during optimization and will lower over time as the algorithm reaches an asymptotic regime where using mini-batches of size~$B$ will have the same sample efficiency
as without mini-batches.

In this paper, we make the following contributions:

\begin{itemize}

\item

We propose in \mysec{adabatch} a new merging operator for gradients computed over a mini-batch, to replace taking the average. Instead, for each mini-batch we count for each coordinate how many samples
had a non zero gradient in that direction. Rather than taking the sum of all gradients and dividing by $B$ we instead divide each coordinate independently by the number of times it was
non zero in the batch.
This happens to be equivalent to reconditioning the initial problem in order to exploit its sparsity. Because each coordinate is still
an average (albeit a stochastic one), the norm of the gradient
will stay under control. In order to notice this effect, one has to look at the problem in a different geometry that accounts for the sparsity of the data. We also draw a parallel with Adagrad-type methods \cite{adagrad,tonga} as our operator is equivalent to an implicit rescaling of the gradients per coordinate.

\item
We show in \mysec{convergence} that this new merging rule  outperforms regular mini-batch on sparse data and
that it can have the same if not an improved sample efficiency compared to regular SGD without mini-batch.

\item
We  explain in \mysec{wild} how this can be used to make synchronous parallel or distributed methods able to compete with asynchronous ones while being easier to study as they are logically equivalent to the sequential version.

\item
We extend our results to variance-reduced SG methods such as SVRG in \mysec{svrg} and show similar gains are obtained when using AdaBatch.

\item We  present in \mysec{exp} experimental results to support our theoretical claims as well
as a proof of concept that our new merging operator can make synchronous parallel SG methods
competitive with asynchronous ones. We also extend our experiments to variance reduction methods like SVRG and show that AdaBatch yields similar improvement as in the case of SG methods.

\end{itemize}

\textbf{Notations.} Throughout this paper, $\| \cdot \|$ denotes the Euclidean norm on~$\reel^d$ and for any symmetric positive definite matrix $D $, $\norm{\cdot}_D$  is  the norm defined by $D$ so that $\forall x \in \reel^d, \norm{x}_D^2 = x^T D x$; for a set $A$, $\abs{A}$  denotes the cardinality of $A$.
If $x \in \reel^d$ then $x^{(k)}$  denotes the $k$-th coordinate and $\Diag{x}$ is the $d\times d$ diagonal matrix with
the coefficient of $x$ on its diagonal. We define for any integer $n$, $[n] := \{1, 2, \ldots, n\}$.
Finally,  for any function $h : \reel^d \to \reel $, we will define the support of $h$ as
\begin{align}
\begin{split}
    \Suppo{h} = \{ k \in [d]: &\exists (x,y) \in \reel^d\times \reel^d,
    x^{(k)} \neq y^{(k)} \;\text{and}\; h(x) \neq h(y)\}.
\end{split}
\end{align}

\subsection{Problem setup}
\label{sec:setup}
We consider $f$ a random variable with values in the space of convex functions $\mathcal{F}$ from $\reel^d$ to $\reel$.
We define $F(w) := \esp{f(w)}$ and we wish to solve the optimization problem
\begin{equation}
    \label{eq:optim_problem_general}
    F_* = \min_{w \in \reel^d} F(w).
\end{equation}
It should be noted that the gradient $f'$ will only have non zero coordinate along the directions of the support $\Suppo{f}$ so that if the support of $f$ is sparse, so will the update for SG methods.
We define $p \in \reel^d$ by $\forall k \in [d], p^{(k)} := \prob{k \in \Suppo{f}}$. We take $\pmin:=\min(p)$ and $\pmax=\max(p)$.

This setup covers many practical cases, such as finite sum optimization where~$f$ would have
the uniform distribution over the sum elements or stochastic online learning where $f$ would be an infinite
stream of training samples.

One example of possible values for $f$ is given by linear predictions with sparse features. Let us assume $X$ is a random variable with values in $\reel^d$ and $\phi : \reel \rightarrow \reel$ a random
convex function.
Then one can take $f(w) := \phi(X^T w)$;
$\Suppo{f}$ would be the same as $\Suppo{X}$ defined as the non zero coordinates of the vector $X$.
The problem given by \eqref{eq:optim_problem_general} becomes
\begin{equation}
    \label{eq:optim_problem_linear}
    F_* = \min_{w\in \reel^d} \esp{\phi(X^T w)}.
\end{equation}
In the case of logistic regression,
one would have for instance $\phi(X^T w) = \log(1 + \exp{-Y X^T w })$ for $Y\in\{-1, 1\}$ the random label associated with the feature vector $X$.

The convergence properties of SG methods depend on the properties of the Hessian $F''$ of our objective function $F$,
as we will show in \mysec{convergence}. The closer it is to identity, the faster SG methods will converge
and this convergence will be as fast for all the coordinates of $w$. For example, in the case of sparse linear prediction such
as given by~\eqref{eq:optim_problem_linear}, with binary features $X^{(k)} \in \{0, 1\}$ that are uncorrelated, the Hessian is such that $F''(w) = \esp{\phi''(X^T w) X X^T}.$
If there exist $M$ and $m$ so that we have $\forall z \in \reel, m \leq \phi''(z) \leq M$, then we immediately have
\begin{align}
\label{linear_hessian_F}
\begin{split}
   \!\!\! \forall w\in \reel^d,\  &m (1 - \pmax) \Diag{p} \preceq F''(w), \\
        &F''(w) \preceq M \big( \textstyle 1 + \sum_{k\in [d]} p^{(k)}\big) \Diag{p}.
\end{split}
\end{align}
We notice here that we have a specific structure to the geometry of $F$ which depends on $p$ and which need to be taken
into account. The proof of~\eqref{linear_hessian_F} is given in the supplementary material (\mysecan{sparse_linear_prediction}).

Finally, we want not only to solve problems \eqref{eq:optim_problem_general} or \eqref{eq:optim_problem_linear}, but to be able to do so while using $W$ workers. Those workers can either be running on the same machine with shared memory or in a distributed fashion.

\subsection{Related work}

There have been several approaches for parallelizing or distributing SG methods.

\textbf{Parallelized stochastic gradient descent.}
This approach described by \cite{parallel_sgd_average} consists in splitting a dataset in $W$ different parts
and averaging the model obtained by $W$ independent workers. Model averaging always reduces
the variance of the final estimator but the impact on the bias is not as clear. For sparse optimization this approach does not
in practice outperform a purely sequential algorithm~\cite{Hogwild}.

\textbf{Delayed stochastic gradient descent.} The effect of delay for constant step-size stochastic gradient descent has been studied by \cite{slow_learners_are_fast}. Allowing for delay will remove the need for synchronization and thus limit the overhead
when parallelizing. The main result of \cite{slow_learners_are_fast} concludes that there is two different regimes. During the first phase,  delay will not help convergence, although once the asymptotic terms are dominating,  a theoretical linear speedup with the number of worker is recovered.

\textbf{Using mini-batches}
is a popular alternative for parallelizing or distributing SGD.
In \cite{online_minibatch}, the reduction of the variance of the gradient estimate is used to prove improvement in convergence. Our theoretical and practical results show that in the case of sparse learning, mini-batch do not offer an improvement during the first stage of optimization. We believe our merging rule is a simple modification of mini-batch SGD that can considerably improve convergence speed compared to regular mini-batch.

The case of averaged stochastic gradient descent with constant step size for least-squares regression has been
studied in much detail and in that case it is possible to get an explicit expression
for the convergence of the algorithm~\cite{minibatch}.
During the first phase of optimization, in order to achieve the same
accuracy after a given number of samples, the step size must be increased proportionally to the batch size which is possible up to a point after which the algorithm will diverge. We draw the same conclusions in a more generic case in \mysec{convergence}.

In \cite{minibatch_sgd}, a specific subproblem is solved instead of just averaging the gradients in a mini-batch. However, solving a subproblem is much more complex to put in place and requires the tuning of extra parameters.
Our method is very simple as it only requires a per-coordinate rescaling of the gradients and does not require any parameter tuning.

\textbf{Hogwild!} is a very simple parallel SG method. Each worker processes training examples completely in parallel, with no synchronization and accessing the same model in memory \cite{Hogwild}. The overhead is minimal, however the theoretical analysis
is complex.
New proof techniques have been introduced to tackle
those issues \cite{hogwild_mania}.
Our contribution here is to make synchronous methods almost as fast as Hogwild!. This has an interest for cases where Hogwild! cannot perform optimally, for instance with a mixture of dense and sparse features, or in the distributed setting where memory cannot be shared. Hogwild! has inspired parallel versions for SDCA, SVRG and SAGA \cite{passcode,hogwild_mania,asaga}. AdaBatch can similarly be extended to those algorithms and we provide proof for SVRG.
 \textbf{Cyclades} \cite{cyclades} builds on Hogwild!, assigning training samples
to specific workers using graph theory results to remove conflicts.

\textbf{Adagrad.} Adagrad~\cite{adagrad} performs a per coordinate rescaling dependent on the size of past gradients that has proven to be highly efficient
for sparse problem, besides it can be combined with Hogwild! for parallel optimization \cite{duchi2013estimation}. Adagrad rescaling is similar in nature to the one performed by AdaBatch. Adagrad has a step size going to 0 with the number of iterations which
gives  good convergence properties for various problems. On the other hand, AdaBatch works with a wider
range of methods such as SVRG. Constant step size has proven useful for least-mean-square problems \cite{bach13} or in the field of deep learning \cite{sutskever2013importance}.

\section{AdaBatch for SGD}
\label{sec:adabatch}

In this section, we will focus on constant step size stochastic gradient descent to give
an intuition on how AdaBatch works. AdaBatch can be extended in the same way to SVRG
(see~\mysec{svrg}).
We assume we are given a starting point $w_0 \in \reel^d$ and we define recursively
\begin{align}
\label{eq:recursion}
    \forall n > 0, w_{n} &= w_{n-1} - g_n,
\end{align}
for a sequence $g_n$ of stochastic gradient estimates based on independent gradients $f'_{n,1}, \dots,f'_{n,B}$. We define $\gmb{n}$ as
\begin{align}
\label{eq:definition_gmb}
    \forall k \in [d], \ \gmb{n}^{(k)} &= \frac{1}{B} \sum_{b : k\in \Suppo{f}} f'_{n,b}(w_{n-1})^{(k)} .
\end{align}
Plugging \eqref{eq:definition_gmb} into \eqref{eq:recursion} yields the regular SGD mini-batch
algorithm with constant step size $\gamma$ and batch size $B$.

For each iteration $n > 0$ and dimension $k \in [d]$,
we denote $D_{n,k} := \{b \in [B] : k \in \Suppo{f_{n, b}}\}$. We introduce $\gab{n}$, the gradient estimate of AdaBatch, as,  $\forall k \in [d]$,
\begin{align}
\label{eq:definition_gab}
    \gab{n}^{(k)}  &=
        \begin{cases}
            \frac{\sum_{b \in D_{n,k}} f'_{n,b}(w_{n-1})^{(k)} }{
                \abs{D_{n,k}}} \quad &\text{if $D_{n,k} \neq \varnothing$,}\\
            0 \quad &\text{otherwise}.
        \end{cases}
\end{align}
Instead of taking the average of all gradients, for each coordinate we make an average but taking only
the non zero gradients into account. Let us take a coordinate $k\in[d]$; if $p^{(k)}$ is close to 1, then
the AdaBatch update for this coordinate will be the same as with regular mini-batch with high probability.
On the other hand, if $p^{(k)}$ is close to zero, the update of AdaBatch will be close to summing the gradients
instead of averaging them. Adding updates instead of averaging has been shown to be beneficial
in previous work such as in CoCoa+~\cite{cocoa}, a distributed SDCA-inspired optimization
algorithm. We observed experimentally that in order to achieve the same performance
with mini-batch compared to AdaBatch, one has to take a step size that is proportional to $B$. This will boost convergence for less frequent features but can lead to divergence
when $B$ increases because the gradient for frequent features will get too large.
Our method allows to automatically and smoothly move from summing to averaging
depending on how frequent a feature is.

Let us consider the expectation for those two updates rules, we have to use the expectation of $\gmb{n}$ and $\gab{n}$. We have immediately $\esp{\gmb{n}} = F'(w_{n-1})$.
Thus when using the regular mini-batch update rule, one obtains an unbiased
estimate of the gradient. The main advantage of mini-batch is a reduction by a factor $B$
of the stochastic noise near the optimal value, as explained in \mysec{convergence}. With our new rule, using Lemma~1 from the supplementary material (with divisions of probabilites taken element-wise), we have:
\begin{align}
\label{eq:esp_gab}
    \esp{\gab{n}} &= \Diag{\frac{1-(1-p)^B}{p}} F'(w_{n-1}).
\end{align}
Interestingly, we now have a gradient that is equivalent to a reconditioning of $F'$.
We can draw here a parallel with Adagrad \cite{adagrad} which similarly uses per-coordinate step sizes. When using Adagrad, the update rule becomes
\begin{align*}
    w_{n} &= w_{n-1} - \gamma (  C_{n}^{\mathrm{adag}}  ) ^{-1} f'_n(w_{n-1})
 \mbox{ with }\\
    C_{n}^{\mathrm{adag}} &\!=\! \Diag{\alpha^{-1}\textstyle \sqrt{\epsilon + \sum_{i \in [n-1]} (f'_i(w_{i-1})^{(k)})^2}}_{k\in[d]}.
\end{align*}
The goal is to have an adaptative step size that will have a larger step size
for coordinate for which the gradients have a smaller magnitude.
One should note a few differences though:
\begin{itemize}
    \item Adagrad relies on past informations and updates the reconditioning at every iteration. It works without any particular requirement on the problem.
    Adagrad also forces a decaying step size. Although this give Adagrad  good convergence properties, it is not always suitable, for instance when using variance reduction methods such as SVRG where the step size is constant.
    \item On the other side, the AdaBatch scaling stays the same (in expectation) through time and only tries to exploit the structure
    coming from the sparsity of the problem. It does not require storing extra information
    and can be adapted to other algorithms such as SVRG or SAGA.
\end{itemize}
When used together with mini-batch, Adagrad will act similarly to AdaBatch
and maintain the same sample efficiency automatically even when increasing the batch size.

\paragraph{Deterministic preconditioning.} One can see that when $B\rightarrow \infty$, the reconditioning in \eqref{eq:esp_gab} goes to $\Diag{p}^{-1}$.
One could think of directly scaling the gradient by $\Diag{p}^{-1}$ to achieve a tighter bound in~\eqref{linear_hessian_F}.
However this would make the variance of the gradient explodes and thus is not usable in practice as shown in~\mysec{convergence}. A key feature of our update rule is stability; because every coordinate of $\gab{n}$ is an
average, there is no chance it can diverge.

Using directly $\Diag{p}^{-1}$ is not possible, however reconditioning by $C_{B, p} := \Diag{\frac{1-(1-p)^B}{p}}$
when using a mini-batch of size $B$ might just work as it will lead to the same expectation as~\eqref{eq:definition_gab}.
If we define $\grc{n} := C_{B, p} f'(w_{n-1})$, we immediately have $\esp{\grc{n}} = C_{B, p} F'(w_{n-1})$.
Intuitively, $C_{B, p}$ is getting us as close as possible to
the ideal reconditioning $\Diag{p}^{-1}$ leveraging the mini-batch size in order to keep the size of the gradients under control.
Both $\grc{n}$ and $\gab{n}$ allow to obtain a very similar performance both in theory and in practice,
so that which version to choose will depend on the specific task to solve.
If it is possible to precompute the probabilities~$p$ then one can use $\grc{n}$
which has the advantages of not requiring the extra step of counting the features
present in a batch. On the other hand, using $\gab{n}$ allows to automatically
perform the same reconditioning with no prior knowledge of $p$.

\section{Convergence results}
\label{sec:convergence}

We make the following assumptions which generalize our observation from \mysec{setup} for sparse linear prediction.
\begin{assumption}
\label{assumptions}
    We assume there exists  a convex compact set $\mathcal{D}\subset \reel^d$, $\mu$, $L$ and $R$ strictly positive so that   the following assumptions are satisfied.
    \begin{enumerate}
        \item The Hessians $F''$ (resp.~$f''$) of $F$ (resp.~$f$) are such that:

        \begin{equation}
        \begin{split}
            &\forall w \in \mathcal{D}, \ \  \mu \Diag{p} \preceq F''(w) \preceq L \Diag{p},
            \quad \text{and}\\
            &\forall w \in \mathcal{D},
            f''(w) \preceq R^2 \Id.\\[-.5cm]
        \end{split}
        \end{equation}
        \item Let $w_* := \argmin_{w \in \mathcal{D}} F(w)$,
        \begin{align}
        \label{eq:global_minimizer}
            &F'(w_*) = 0.
        \end{align}
        In particular, $w_*$ is a global minimizer of $F$ over $\reel^d$.
    \end{enumerate}
\end{assumption}

Those assumptions are easily met in the case of sparse linear predictions. If $f(w) := \phi(X^T w)$, with $\forall k\in[d], \  X^{(k)}\in \{0, 1\}$ uncorrelated, $\norm{X}^2 \leq G^2$ almost surely  and $m \leq \phi'' \leq M$, then the assumptions above are verified
for $L = M (1 + \sum_{k\in[d]}{p^{(k)}})$, $\mu = m (1 - \pmax)$, and $R^2 = G^2 M$. In the case of the logistic  loss, $M := 1/4$ and $\mu$ typically exist on any compact but is not explicitly available. Note though that it is not required to know $\mu$  in order to train any of the algorithms studied here. More details are given in the supplementary material (\mysecan{sparse_linear_prediction}).

Detailed proofs of the following results are given in the supplementary material (\mysecan{proof_adabatch_sgd}). Our proof technique is based on a variation from the one introduced by \cite{proof_sgd}.
It requires an extra projection step on $\mathcal{D}$. In practice however, we did not require it for any reasonable step size that does not make the algorithm diverge and our bounds do not depend on $\mathcal{D}$ because of~\eqref{eq:global_minimizer}.
The results are summarized in Table \ref{table:results}.
We use a constant step size as a convenience for comparing the different algorithms. It is different but equivalent to using a decreasing step size (see \cite{nesterov2008confidence} just before Corollary~1).
For instance, if we know the total number of iterations is $n \gg 1$, taking $\gamma = \frac{6 \log(n)}{n}$, the bias term is
approximately $\frac{\mu}{ 4 \log(n) n^2}$. The variance term which is proportional to $\gamma$ is a $O(\log(n)/n)$ which is the usual rate for strongly convex~SGD.

\begin{table*}
\begin{center}
  \setlength\extrarowheight{7pt}
\begin{tabular}{|l|ll|l|}
  \hline
  \textbf{Method} & {\boldmath $F_{N/B} - F_*$} & & {\boldmath $\gmax$} \\
  \hline
  Mini-batch & $(1 - \gamma \pmin \mu /2)^{N/B} \frac{\delta_0}{\gamma} $ &$+ \gamma \frac{2 \sigma^2}{B}$ & $\gamma \left[L \pmax + \frac{2 R^2}{B}\right] \leq 1$ \\
  \hline
  AdaBatch & $(1 - \gamma \pmin^{+B} \mu /2)^{N/B} \frac{\delta_0}{\gamma}$ &$+2 \gamma \sigma^2$ & $\gamma \left[L + 2 R^2\right] \leq 1$ \\
  \hline
  $\Diag{p}^{-1}$ &$(1 - \gamma \mu /2)^{N/B} \frac{\delta_0}{\gamma}$ &$+ \gamma\frac{2 \sigma^2}{\pmin B}$ & $\gamma \left[L + \frac{2 R^2}{\pmin B}\right] \leq 1$ \\
  \hline
\end{tabular}
\end{center}

\caption{Convergence rates for the different methods introduced in  \mysec{adabatch}. $N$ represents the total number of samples, so that the number of iterations is~$N / B$; $\gmax$ is the maximum step size that guarantees this convergence, $\sigma^2 := \esp{\norm{f'(w_*)}^2}$ the gradient variance at the optimum, and $\forall p\in[0, 1], p^{+B} := 1 - (1 - p)^B$. The $\Diag{p}^{-1}$ method consists
    in reconditioning by $\Diag{p}^{-1}$. Moreover, $\delta_0 := \norm{w_0 - w_*}_A^2$
    where $A = \Id$ for mini-batch SGD, $A = \Diag{p^{+B}/p}$ for AdaBatch
    and $A = \Diag{p}^{-1}$ for the last method.}
\label{table:results}
\end{table*}

\paragraph{Bias/variance decomposition.} We notice that the final error is made of two terms, one that decreases exponentially fast and measures how quickly we move away from the starting point and one that
is constant, proportional to~$\gamma$ and that depends on the stochastic noise around the optimal value. We will call the former the \emph{bias} term and the latter the \emph{variance} term, following the terminology introduced by~\cite{bach13}.
The bias term decreases exponentially fast and will be especially important during the early stage of optimization and
when $\mu$ is very small. The variance term is the asymptotically dominant term and
will prevail when close to the optimum.
In practical applications, the bias
term can be the most important one to optimize for~\cite{defossez}. This has also been observed for deep learning, where most of the optimization is spent far from the optimum~\cite{sutskever2013importance}.

We immediately notice that rescaling gradients by $\Diag{p}^{-1}$
is infeasible in practice unless $B$ is taken of the order of $\pmin^{-1}$, because of the exploding variance term and the tiny step size.

\paragraph{Mini-batch.} Let us now study the results for mini-batch SGD. The variance term is always improved by a factor of~$B$ if we keep the same step size. Most previous works on mini-batch
only studied this asymptotic term and concluded that because of this linear scaling,
mini-batch was efficient for parallel optimization \cite{online_minibatch}.
However, when increasing the batch size, the number of iterations is divided by $B$ but the exponential rate is still the same. The bias term will thus not converge
as fast unless we increase the step size. We can see two regimes depending on $B$.
If $B \ll \frac{2 R^2}{L \pmax}$, then the constraint is $\gamma \leq \frac{B}{2 R^2}$. Thus, we can scale $\gamma$ linearly and achieve the same convergence for both the variance and bias term as when $B=1$.
However, if $B \gg \frac{2 R^2}{L \pmax}$, then the constraint is $\gamma \leq \frac{1}{L \pmax}$. In this regime, it is not possible to scale up infinitely $\gamma$ and thus it is not possible to achieve the same convergence for the bias term as when $B=1$. We have observed this in practice on some datasets.

\paragraph{AdaBatch.} With AdaBatch though, if $\pmin \ll 1$, then $1 - (1-\pmin)^B \approx B \pmin$. In such case, the bias term is
$(1 - \gamma B \pmin \mu /2)^{N / B} \frac{\delta_0}{\gamma} \approx (1-\gamma \pmin \mu /2) ^ {N} \frac{\delta_0}{\gamma}$,
thus showing that we can achieve the same convergence speed for a given number of samples as when $B = 1$ as far as the bias term is concerned.
The variance term and the maximum step size are exactly the same as when $B=1$. Thus, as long as
$1 - (1-\pmin)^B \approx B \pmin$, AdaBatch is able to achieve at least the same sample efficiency as when $B=1$. This is a worst case scenario, in practice we observe on some datasets an
improved efficiency when increasing $B$, see~\mysec{exp}. Indeed for rare features
the variance of the gradient estimate will not be decreased with larger batch-size, however
for features that are likely to appear more than once in a batch, AdaBatch will
still obtain partial variance reduction through averaging more than one gradient. Although
we do not provide the full proof of this fact, this is a consequence of Lemma 2 from the supplementary material.

\section{Wild AdaBatch}
\label{sec:wild}
We now have a SG method trick that allows us to increase the size of mini-batches
while retaining the same sample efficiency.
Intuitively one can think of sample efficiency as how much information
we extract from each training example we process i.e. how much the loss will decrease
after a given number of samples have been processed.
Although it is easy to increase the number of samples processed per seconds
when doing parallel optimization, this will only lead to a true speedup
if we can retain the same sample efficiency as sequential SGD. If the sample efficiency
get worse, for instance when using regular mini-batch, then we will have to perform
more iterations to reach the same accuracy, potentially canceling out the gain obtained
from parallelization.

When using synchronous parallel SG methods such as \cite{online_minibatch}, using large mini-batches allows to reduce the overhead and thus increase the number of samples processed per second. SGD with mini-batches typically suffers from a lower sample efficiency when $B$ increases. It has been shown to be asymptotically optimal \cite{minibatch_sgd,online_minibatch}, however our results summarized in \mysec{convergence} show that in cases where the step size cannot be taken too large, it will not be able to achieve the same sample efficiency as SGD without mini-batch.

We have shown in~\mysec{convergence} that for sparse linear prediction, AdaBatch can achieve the same sample efficiency as for $B=1$. Therefore, we believe it is a better
candidate than regular mini-batch for parallel SGD. In \mysec{exp}, we will present our experimental results
for Wild AdaBatch, a Hogwild! inspired, synchronous SGD algorithm. Given a batch size $B$ and $W$ workers,
they will first compute in parallel $B$ gradients, wait for everyone to be done and then apply the updates in parallel to update $w_n$. The key advantage here is that thanks to the synchronization, analysis of this algorithm is easier. We have a clear sequence of iterates $w_n$ in memory and there is no delay or inconsistent read. It is still possible that during the update phase, some updates will be dropped because of overlapping writes to memory, but that can be seen simply as a slight decrease of the step size for those coordinate. In practice, we did not notice any difference with the sequential version of AdaBatch.

\section{AdaBatch for SVRG}
\label{sec:svrg}

SVRG~\cite{svrg} is a variance-reduced SG
method that has a linear rate of convergence on the training error when
$F$ is given by a finite mean of functions $F := \frac{1}{N} \sum_{i \in [N]} f_i$. This is equivalent to $f$ following the uniform law over the set $\{ f_1, \ldots, f_N\}$.
SVRG is able to converge with a constant step size. To do so, it replaces the gradient $f'(w)$ by $f'(w) - f'(y) + F'(y)$
where $y$ is updated every epoch (an epoch being $m$ iterations where $m$
is a parameter to the algorithm, typically of the order of the number of training samples).
Next, we show the difference between regular mini-batch SVRG and AdaBatch SVRG
and give theoretical results showing improved convergence for the latter.

We now only assume that $F$ verifies the following inequalities for $\mu > 0$, almost surely,
\begin{equation}
\label{eq:hessian_svrg}
    \forall w \in \reel^d, \ \  \mu\Diag{p}\preceq F''(w) \  \text{and} \   f(w) \preceq L\Diag{p}.
\end{equation}

Let us take
a starting point $y_0 \in \reel^d$ and $m \in \mathbb{N}^*$. For all $s = 0, 1, \ldots$, we have
$w_{s, 0} := y_s$ and for all $n \in [m]$ let us define
\begin{align*}
    w_{s, n} := w_{s, n-1} - \gamma g_{s,n}, \quad y_{s + 1} := \frac{1}{m} \sum_{n\in[m]} w_{s, n},
\end{align*}
with $g_{s,n}$ the SVRG update based on $(f_{s,n,b})_{b\in[B]}$ i.i.d.~samples of $f$.
Let us introduce
    \[
    \forall k \in [d], D^{(k)}_{s, n} := \left\{b \in [B]: k \in \Suppo{f_{s,n,b}'}\right\}.
    \]
 For any dimension $k$ such that $D^{(k)}_{s, n}\neq \varnothing$ we have
\begin{align*}
g_{s,n}^{(k)} := \frac{1}{C_{s,n}^{(k)}}\Big(
\sum\limits_{b \in D^{(k)}_{s,n}} &f_{s,n,b}'(w_{s, n-1})^{(k)} - f_{s,n,b}'(y_s)^{(k)}
 + F'(y_s)^{(k)} / p^{(k)}\Big),
\end{align*}
and for $D^{(k)}_{s, n}= \varnothing$ we take $g_{s,n}^{(k)} := 0$.
Regular mini-batch SVRG is recovered for $C_{s,n}^{(k)} := B$.
On the other hand, AdaBatch SVRG is obtained for $C_{s,n}^{(k)} := \abs{D^{(k)}_{s,n}}$.
One can note that we
used a similar trick to the one in~\cite{hogwild_mania} in order to preserve
the sparsity of the updates.

For both updates, there exists a choice of $\gamma$ and $m$ such that \begin{align*}
    \esp{F(y_s) - F_*} \leq 0.9^s (F(y_0) - F_*).
\end{align*}
For regular mini-batch, it is provably sufficient to take
$\gamma_{\textrm{mb}} = \frac{1}{L}$ and $m_{\textrm{mb}} \approx \frac{2.2 L}{\pmin \mu}$.
For AdaBatch, we have
$\gamma_{\textrm{ab}} = \frac{1}{10 L}$ and $m_{\textrm{ab}} \approx \frac{20 L}{B \pmin \mu}$.
We notice that as we increase the batch size, we require the same number of inner iterations
when using regular mini-batch update. However, each update requires $B$ times more
samples as when $B = 1$. On the other hand when using AdaBatch, $m_{\textrm{ab}}$ is
inversely proportional to $B$ so that the total number of samples required to reach the same accuracy
is the same as when $B  = 1$.
We provide detailed results and proofs in the supplementary material (\mysecan{proof_adabatch_svrg}). Note that similar results should hold for epoch-free variance-reduced SG methods such as SAGA~\cite{defazio2014saga}.

\section{Experimental results}
\label{sec:exp}

We have implemented both Hogwild! and Wild AdaBatch and compared them on three datasets, \emph{spam}\footnote{\url{http://plg.uwaterloo.ca/~gvcormac/trecspamtrack05/trecspam05paper.pdf}},
 \emph{news20}\footnote{\url{https://www.csie.ntu.edu.tw/~cjlin/libsvmtools/datasets/binary.html}},
 and \emph{url} \cite{url}.
\emph{Spam} has 92,189 samples of dimension 823,470 and an average of 155 active features per sample. \emph{News20} has 19,996 samples of dimension 1,355,191 and an average of 455 active features per example. Finally, \emph{url} has 2,396,130 samples of dimension 3,231,961 and an average of 115 active features per example.

We implemented both in C++ and tried our best to optimize both methods. We ran them on an Intel(R) Xeon(R) CPU E5-2680 v3 @ 2.50GHz
with 24 CPUs divided in two sockets. Each socket contains 12 physical CPUs for 24 virtual ones. In our experiments though, it is better to keep the number of threads under the number of actual physical CPUs on a single socket. We restricted each experiment to run on a single socket in order to prevent NUMA (non uniform memory access) issues.
For all the experiments we ran (and all methods), we perform a grid search to optimize the step size (or the main parameter of Adagrad).
We then report the test error on a separate test set.

\textbf{Wild AdaBatch.}
 We trained each algorithm for logistic regression with 5 passes over the dataset, 1 pass only in the case of \emph{url}. We normalized the features so that each sample has norm 1.
For each dataset, we evaluate 3 methods: Wild AdaBatch (AB) and Wild mini-batch SGD (MB, the same as AB but with the regular average of the gradients) as well as Hogwild! (HW) for various numbers of workers $W$. For AdaBatch and mini-batch SGD, the batch size is set to $B = 10 W$ for \emph{news20} and $B = 50$ for \emph{url} and \emph{spam} which was giving a better speedup for those datasets. We take $W$ going from $1$ to $12$, which is the maximum number of physical CPUs on a single socket on our machine.
We also evaluate purely sequential SGD without mini-batch (SEQ).

We only present here the result for \emph{news20}. The figures for the other datasets can be found in the supplementary material~\mysecan{experimental_results}.
In Figure~\ref{news20:conv:nonstrict} we show the convergence as a function
of the wall-clock time.
In Figure~\ref{time_to_target_ips}, we give the wall-clock time to reach a given test error (where our method is achieving close to a linear speed-up) and the number of processed samples per seconds. On \emph{news20}, the gain in sample efficiency actually allows Wild AdaBatch to reach the goal the fastest, even though Wild AdaBatch can process
less samples per seconds than Hogwild!, thanks to its improved sampled efficiency.

\textbf{Comparison with Adagrad.}
We compare AdaBatch with Adagrad for various batch-size on Figure~\ref{figure:adagrad} when trained with a fixed number of samples, so that when $B$ increases, we perform less iterations.
On the \emph{url} dataset, Adagrad performs significantly better than AdaBatch, however we notice that as the batch-size
increases, the gap between AdaBatch and Adagrad reduces.
On the \emph{spam} dataset with the least-mean-square loss, constant step size SGD performs better
than Adagrad. We believe this is because Adagrad is especially well suited for non strictly convex problems.
For strictly convex problem though, constant step size SGD is known to be very efficient~\cite{bach13}.
We also plotted the performance of constant step size regular mini-batch SGD. In all cases, regular mini-batch
scales very badly as the batch size increases. We fine tune the step size for each batch size and observed
the regular mini-batch will take a larger step size for small batch sizes, that allows to keep roughly
the same final test error. However, when the batch size increases too much, this is no longer possible as it makes
optimization particularly unstable, thus leading to a clear decrease in sample efficiency.
Finally, AdaBatch can even improve the sample complexity when increasing $B$. We believe this comes from the variance reduction of the gradient for features that occurs more than once in a mini-batch, which in turn allows for a larger
step size.

\textbf{SVRG.}
We also compared the effect of AdaBatch on SVRG. On Figure~\ref{figure:svrg} we
show the training gap $F_N - F^*$ on \emph{url} for the log loss with a small L2 penalty. This penalty is given by $\frac{10^{-4}}{2} \norm{w}_{\diag{p}}^2$, chosen to respect our hypothesis and to
 prevent overfitting without degrading the testing error. All the models are trained with 10 iterations over all the samples in the dataset, so that if $B$ is larger, the model will perform less iterations.
We observe that as we increase the batch size, the sample efficiency of regular mini-batch deteriorates. The variance reduction coming from using a larger mini-batch is not sufficient to compensate the fact that we perform less iterations, even when we increase the step size. On the other hand, AdaBatch actually allows to improve the sample efficiency as it allows to take a larger step size thanks to the gradient variance reduction for the coordinates with $p^{(k)} B$ large enough. To the best of ourknowledge, AdaBatch is the first
mini-batch aggregation rule that is both extremely simple and allows for
a better sample efficiency than sequential SGD.

\begin{figure}
\centering
\begin{minipage}[t]{.45\textwidth}
    \centering
  \begin{tikzpicture}[
    ab/.style={color=blue, solid,mark options={fill=blue}},
    mb/.style={color=red, solid,mark options={fill=red}},
    seq/.style={color=green,solid,mark options={fill=green}},
    hw/.style={color=black, solid,mark options={fill=black}},
    every node/.style={font=\fontsize{8}{5}\selectfont}
  ]
    \begin{loglogaxis}[
        xlabel={time (in sec)},
        ylabel={test error},
        name=plot1,
        yticklabel pos=left,
        legend entries={
            {{AB $W=2$},
             {AB $W=6$},
             {MB $W=2$},
             {MB $W=6$},
             {SEQ $W=1$},
             {HW $W=2$},
             {HW $W=6$},
             }
        },
        legend style={
            font=\fontsize{5}{5}\selectfont,at={(0.5,1.02)},anchor=south,
            legend columns=3},
        style={font=\footnotesize},
        width=\linewidth,
        skip coords between index={0}{1}]
        \def\myroot{tables__parallel__news20}

        \pgfplotstableread{\myroot__convergence_depth_total_primal_test_group=ab=10,2.table}\tablea
        \addplot+[ab,mark=*] table[x=x, y=y, y error=y_error] {\tablea};

        \pgfplotstableread{\myroot__convergence_depth_total_primal_test_group=ab=10,6.table}\tablea
        \addplot+[ab, mark=diamond*] table[x=x, y=y, y error=y_error] {\tablea};

        \pgfplotstableread{\myroot__convergence_depth_total_primal_test_group=mb=10,2.table}\tablea
        \addplot+[mb,mark=*] table[x=x, y=y, y error=y_error] {\tablea};

        \pgfplotstableread{\myroot__convergence_depth_total_primal_test_group=mb=10,6.table}\tablea
        \addplot+[mb, mark=diamond*] table[x=x, y=y, y error=y_error] {\tablea};

                \pgfplotstableread{\myroot__convergence_depth_total_primal_test_group=hw=20,1.table}\tablea
        \addplot+[seq, mark=square*] table[x=x, y=y, y error=y_error] {\tablea};

        \pgfplotstableread{\myroot__convergence_depth_total_primal_test_group=hw=20,2.table}\tablea
        \addplot+[hw,mark=*] table[x=x, y=y, y error=y_error] {\tablea};

        \pgfplotstableread{\myroot__convergence_depth_total_primal_test_group=hw=20,6.table}\tablea
        \addplot+[hw,mark=diamond*] table[x=x, y=y, y error=y_error] {\tablea};
    \end{loglogaxis}
\end{tikzpicture}
\captionof{figure}{\label{news20:conv:nonstrict}Convergence result for \emph{news20}.
The error is given  as a function of the wall-clock time.}
\end{minipage}\hfill%
\begin{minipage}[t]{.45\textwidth}
\centering
\begin{tikzpicture}[
    ab/.style={color=blue, solid,mark options={fill=blue}},
    mb/.style={color=red, solid,mark options={fill=red}},
    seq/.style={color=green,solid,mark options={fill=green}},
    hw/.style={color=black, solid,mark options={fill=black}},
    every node/.style={font=\fontsize{8}{5}\selectfont}
  ]
    \begin{semilogyaxis}[
        xlabel={training samples},
        ylabel={training gap $F_N - F^*$},
        name=plot1,
        legend entries={
            {
            {AB/MB $B=1$},
            {AB $B=10$},
            {AB $B=5000$},
            {MB $B=10$},
            {MB $B=5000$},
             }
        },
        legend style={
            font=\fontsize{7}{5}\selectfont,at={(0.5,1.02)},anchor=south,
            legend columns=3},
        style={font=\footnotesize},
        width=\linewidth,
        skip coords between index={12}{14}]
        \newcommand{\myplot}[1]{%
        \pgfplotstableread{tables__svrg__url__convergence_total_gap_group=#1.table}\tablea}

        \myplot{ab=1}
        \addplot+[hw,mark=*] table[x=x, y=y, y error=y_error] {\tablea};
        \myplot{ab=10}
        \addplot+[ab,mark=diamond*] table[x=x, y=y, y error=y_error] {\tablea};
        \myplot{ab=5000}
        \addplot+[ab,mark=square*] table[x=x, y=y, y error=y_error] {\tablea};

        \myplot{mb=10}
        \addplot+[mb,mark=diamond*] table[x=x, y=y, y error=y_error] {\tablea};
        \myplot{mb=5000}
        \addplot+[mb,mark=square*] table[x=x, y=y, y error=y_error] {\tablea};
    \end{semilogyaxis}
\end{tikzpicture}
\captionof{figure}{\label{figure:svrg} Comparison of the training gap $F_n - F_*$ for regular mini-batch vs. AdaBatch
with SVRG on \emph{url} for the log loss with an L2 penalty of $\frac{10^{-4}}{2} \norm{w}_{\diag{p}}^2$.}
\end{minipage}
\end{figure}

\begin{figure}
\centering
\begin{minipage}[t]{0.45\textwidth}
\centering
\begin{tikzpicture}[
    ab/.style={color=blue, solid,mark options={fill=blue}},
    mb/.style={color=red, solid,mark options={fill=red}},
    seq/.style={color=green,solid,mark options={fill=green}},
    hw/.style={color=black, solid,mark options={fill=black}},
    every node/.style={font=\fontsize{8}{5}\selectfont}
  ]
    \begin{axis}[
        xlabel={$W$},
        ylabel={time (in sec)},
        name=plot1,
        legend entries={
            {{AB},
             {AB ideal},
             {MB},
             {HW},
             {HW ideal}
             }
        },
        legend style={
            font=\fontsize{7}{5}\selectfont,at={(0.5,1.2)},anchor=north,
            legend columns=3},
        style={font=\footnotesize},
        width=\linewidth]
        \def\myroot{tables__parallel__news20}

        \pgfplotstableread{\myroot__time_to_target_group=ab=10.table}\tablea
        \addplot+[ab,mark=square*] table[x=x, y=y, y error=y_error] {\tablea};
        \addplot[ab,dashed] table[x=x, y expr=0.080099/\thisrow{x}, y error=y_error] {\tablea};

        \pgfplotstableread{\myroot__time_to_target_group=mb=10.table}\tablea
        \addplot+[mb, mark=diamond*] table[x=x, y=y, y error=y_error] {\tablea};

        \pgfplotstableread{\myroot__time_to_target_group=hw=20.table}\tablea
        \addplot+[hw,mark=*] table[x=x, y=y, y error=y_error] {\tablea};
        \addplot[hw,dashed] table[x=x, y expr=0.119336/\thisrow{x}, y error=y_error] {\tablea};
    \end{axis}
\end{tikzpicture}
\end{minipage}\hfill%
\begin{minipage}[t]{.45\textwidth}
\begin{tikzpicture}[
    ab/.style={color=blue, solid,mark options={fill=blue}},
    mb/.style={color=red, solid,mark options={fill=red}},
    seq/.style={color=green,solid,mark options={fill=green}},
    hw/.style={color=black, solid,mark options={fill=black}},
    every node/.style={font=\fontsize{8}{5}\selectfont}
  ]
    \begin{axis}[
        xlabel={$W$},
        ylabel={samples / sec},
        legend entries={
            {{AB},
             {MB},
             {HW},
             }
        },
        legend style={
            font=\fontsize{7}{5}\selectfont,at={(0.05,0.95)},anchor=north west,
            legend columns=3},
        style={font=\footnotesize},
        width=\linewidth]
        \pgfplotstableread{tables__parallel__news20__speedup_group=ab=10.table}\tablea
        \addplot+[ab,mark=square*] table[x=x, y=y, y error=y_error] {\tablea};
        \pgfplotstableread{tables__parallel__news20__speedup_group=ab=10.table}\tablea

        \pgfplotstableread{tables__parallel__news20__speedup_group=mb=10.table}\tablea
        \addplot+[mb, mark=diamond*] table[x=x, y=y, y error=y_error] {\tablea};

        \pgfplotstableread{tables__parallel__news20__speedup_group=hw=20.table}\tablea
        \addplot+[hw,mark=*] table[x=x, y=y, y error=y_error] {\tablea};
    \end{axis}
\end{tikzpicture}
\end{minipage}
\caption{\label{time_to_target_ips} \textbf{Left}: time to achieve a given test error when varying
the number of workers on \emph{news20}. The dashed line is the ideal speedup.
\textbf{Right}: number of process sampled per second as a function of $W$.}
\end{figure}

\begin{figure}
\centering
\begin{minipage}[t]{0.45\textwidth}
\centering
\begin{tikzpicture}[
    ab/.style={color=blue, solid,mark options={fill=blue}},
    mb/.style={color=red, solid,mark options={fill=red}},
    seq/.style={color=green,solid,mark options={fill=green}},
    hw/.style={color=black, solid,mark options={fill=black}},
    every node/.style={font=\fontsize{8}{5}\selectfont}
  ]
    \begin{axis}[
        xlabel={$B$},
        ylabel={final test error},
        name=plot1,
        legend entries={
            {{AB},
             {MB},
             {AG}
             }
        },
        legend style={
            font=\fontsize{7}{5}\selectfont,at={(0.5,0.95)},anchor=north,
            legend columns=3},
        style={font=\footnotesize},
        width=\linewidth]
        \def\myroot{tables__adagrad__url}

        \pgfplotstableread{\myroot__final_primal_validate_group=ab.table}\tablea
        \addplot+[ab,mark=square*] table[x=x, y=y, y error=y_error] {\tablea};

        \pgfplotstableread{\myroot__final_primal_validate_group=mb.table}\tablea
        \addplot+[mb, mark=diamond*] table[x=x, y=y, y error=y_error] {\tablea};

        \pgfplotstableread{\myroot__final_primal_validate_group=mb_ag.table}\tablea
        \addplot+[hw, mark=*] table[x=x, y=y, y error=y_error] {\tablea};

    \end{axis}
\end{tikzpicture}
\end{minipage}\hfill%
\begin{minipage}[t]{.45\textwidth}
\begin{tikzpicture}[
    ab/.style={color=blue, solid,mark options={fill=blue}},
    mb/.style={color=red, solid,mark options={fill=red}},
    seq/.style={color=green,solid,mark options={fill=green}},
    hw/.style={color=black, solid,mark options={fill=black}},
    every node/.style={font=\fontsize{8}{5}\selectfont}
  ]
    \begin{axis}[
        xlabel={$B$},
        ylabel={final test error},
        legend entries={
            {{AB},
             {MB},
             {AG}
             }
        },
        legend style={
            font=\fontsize{7}{5}\selectfont,at={(0.5,0.95)},anchor=north,
            legend columns=3},
        style={font=\footnotesize},
        at=(plot1.below south west), anchor=above north west,
        width=\linewidth]
        \def\myroot{tables__adagrad__spam}

        \pgfplotstableread{\myroot__final_primal_validate_group=ab.table}\tablea
        \addplot+[ab,mark=square*] table[x=x, y=y, y error=y_error] {\tablea};

        \pgfplotstableread{\myroot__final_primal_validate_group=mb.table}\tablea
        \addplot+[mb, mark=diamond*] table[x=x, y=y, y error=y_error] {\tablea};

        \pgfplotstableread{\myroot__final_primal_validate_group=mb_ag.table}\tablea
        \addplot+[hw, mark=*] table[x=x, y=y, y error=y_error] {\tablea};
    \end{axis}
\end{tikzpicture}
\end{minipage}
\caption{\label{figure:adagrad} Comparison of Adagrad (AG), AdaBatch (AB) and regular mini-batch (MB). \textbf{Left}: on \emph{url}
for the log loss; \textbf{right}: on \emph{spam} for the least-mean-square loss. We plot the final
test error after the same number of samples (larger batches means less iterations).}
\end{figure}

\section{Conclusion}

We have introduced a new way of merging gradients when using SG methods with mini-batches. We have shown both theoretically and experimentally
that this approach allows to keep the same sample efficiency as when not using any mini-batch and sometimes even improve it.
Thanks to this feature, AdaBatch allowed us to make synchronous parallel SG methods competitive with Hogwild!. Our approach can extend to any SG methods including variance-reduced methods. Although not explored  yet, we also believe that AdaBatch is promising for distributed optimization. In such a case, memory is no longer shared so that Hogwild! cannot be used. Distributed mini-batch or SGD with delay have been used in such case \cite{online_minibatch,slow_learners_are_fast}; AdaBatch is a few line change for distributed mini-batch which could vastly improve the convergence of those methods.

\paragraph{Acknowledgements.}
  We thank Nicolas Flammarion, Timothée Lacroix, Nicolas Usunier and Léon Bottou for interesting discussions related to this work. We acknowledge support from the European Research Council (SEQUOIA project 724063).

\clearpage

\newcommand{\ansection}{\subsection}
\newcommand{\ansubsection}{\subsubsection}
\renewcommand{\thesubsection}{\Alph{subsection}}

\section*{Supplementary material}
\ansection*{Introduction}

We present in \mysecan{algorithms} the pseudocode for AdaBatch and Wild AdaBatch.
In \mysecan{exp_var_adabatch}, we give two lemma from which we can derive the expectation
and variance of the AdaBatch gradient update. In \mysecan{proof_adabatch_sgd}
we study the convergence of regular mini-batch and AdaBatch for SGD as well
as the convergence of reconditionned SGD. In \mysecan{proof_adabatch_svrg}
we compare the convergence of regular mini-batch and AdaBatch for SVRG. Finally in \mysecan{experimental_results} we give convergence plots for Wild AdaBatch and SVRG on the remaining
datasets that were not included in the main paper.

\ansection{Algorithms}
\label{an:sec:algorithms}

We present two possible uses of AdaBatch. Algorithm~\ref{an:algo:adabatch} counts
for each mini-batch the number of time each feature is non zero and use that to
recondition the gradient. This is the algorithm that we study in Section 2 of the main paper.

Algorithm~\ref{an:algo:parallel_adabatch} is an Hogwild!
inspired synchronous SGD method that we introduce in Section 4 of the main paper.
Instead of counting the features, we directly
use the reconditioning $\frac{1-(1-p^{(k)})^B}{p}$ where $B$ is the batch size
and $\forall k \in [d], p^{(k)} = \prob{k \in \Suppo{f}}$ i.e., the probability
that feature $k$ is active in a random training sample.
We prove in section~\ref{an:subsection:reconditioned_sgd} that this reconditioning
benefit from the same
convergence speed as regular AdaBatch and does not require to keep
count of the features which is easier to implement in the parallel setting, although
it requires to precompute the probabilities $p^{(k)}$.

\begin{algorithm}[h]
\caption{AdaBatch}
\label{an:algo:adabatch}
\begin{algorithmic}
    \Function{AdaBatch}{$w_0, N, B, \gamma, f$}
    \For{$n \in [N]$}
        \For{$b \in [B]$}
            \State{Sample $f_{n,b}$ from the distribution of $f$}
            \State{Compute $f'_{n,b}(w)$}
        \EndFor
        \For{$b \in [B]$}
            \For{$k \in \Suppo{f_{n,b}}$}
                \State{\begin{equation}
                \label{an:algo:line:update_rule}
                w_n^{(k)} \gets w_{n-1}^{(k)}- \gamma \frac{f'_{n,b}(w_{n-1})^{(k)}}{
                    \abs{\{b : k \in \Suppo{f_{n, b}}\}}}
                \end{equation}}
            \EndFor
        \EndFor
    \EndFor
    \EndFunction
\end{algorithmic}
\end{algorithm}

\begin{algorithm}[h]
\caption{Wild AdaBatch}
\label{an:algo:parallel_adabatch}
\begin{algorithmic}
    \Function{Wild AdaBatch}{$w_0, N, B, \gamma, p, f$}
    \For{$n \in [N]$}
        \ParFor{$b \in [B]$}
            \State{Sample $f_{n,b}$ from the distribution of $f$}
            \State{Compute $f'_{n,b}(w)$}
        \EndParFor
        \ParFor{$b \in [B]$}
            \For{$k \in \Suppo{f_{n,b}}$}
                \State{\begin{equation}
                w_n^{(k)} \gets w_{n-1}^{(k)}- \frac{\gamma}{B}\frac{1-(1-p^{(k)})^B}{p} f'_{n,b}(w_{n-1})^{(k)}
                \end{equation}}
            \EndFor
        \EndParFor
    \EndFor
    \EndFunction
\end{algorithmic}
\end{algorithm}

\ansection{Expectation and variance of the AdaBatch update}
\label{an:sec:exp_var_adabatch}

The AdaBatch update $\gab{n}$ is defined as
\begin{align}
\label{an:eq:definition_gab}
    \forall k \in [d], \gab{n}^{(k)}  &=
        \begin{cases}
            \frac{\sum_{b : k\in \Suppo{f}} f'_{n,b}(w_{n-1})^{(k)} }{
                \sum_{b : k\in \Suppo{f}} 1} \quad &\text{if $\sum_{b : k\in \Suppo{f}} 1 \neq 0$}\\
            0 \quad &\text{otherwise}.
        \end{cases}
\end{align}
so that we have the recurrence rule for SGD
\begin{align*}
    w_{n} = w_{n-1} -\gamma \gab{n}.
\end{align*}
$\gab{n}$ is a per coordinate stochastic average of only the subset
of the gradients which have a non zero coordinate in that direction.
We will need the following Lemma in order to get the expected value and the variance of $\gab{n}$.
\begin{lemma}
\label{an:lemma:esp_var}
    Let $Z$, $(Z_i)_{i\in [N]}$ $N$ i.i.d.~random variables with value in $\reel$
    for which the set $\{0\}$ is measurable with $p = \prob{Z \neq 0} > 0$, and $A \in \reel$ a random variable
    defined as
    \begin{align*}
        \begin{cases}
            A := 0 \quad \text{if}\quad \forall i \in [N], Z_i = 0\\
            A := \frac{\sum_{i\in [N]} Z_i}{\sum_{i \in [N]: Z_i !=0} 1} \quad \text{otherwise.}
        \end{cases}.
    \end{align*}
    We have
  \begin{equation}
    \label{an:eq:esp_sum_z}
    \esp{A} = \frac{1 - (1- p)^N}{p}\esp{Z},
  \end{equation}\break
  \begin{align}
  \label{an:eq:var_sum_z}
    \esp{A^2} &=
        \frac{(1 - (1 - p)^N)^2}{p^2}\esp{Z}^2
            + \left(
            \sum_{i \in [N]} \binom{N}{i} p^{i} (1-p)^{N - i} \frac{1}{i}
            \right) \left(\frac{\esp{Z^2}}{p} - \frac{\esp{Z}^2}{p^2}\right)\\
    \label{an:eq:var_sum_z_simple}
    &\leq \frac{(1 - (1 - p)^N)^2}{p^2}\esp{Z}^2 + \frac{(1 - (1- p)^N)}{p}
    \esp{Z^2}.
  \end{align}
\end{lemma}
\begin{proof}
    We introduce $\mu$ the measure of $Z$ and $\mu^+$ the measure of $Z^+$ defined for any measurable $A \subset \reel$
    \begin{align*}
        \mu^+(A) = \frac{\mu(A \setminus \{0\})}{\mu(\reel \setminus \{0\})}.
    \end{align*}
    Intuitively $Z^+$ is the random variable we obtain if we drop all realizations where $Z = 0$.
    One can verify that $\esp{Z^+} = \frac{\esp{Z}}{p}$ and $\esp{(Z^+)^2} = \frac{\esp{Z^2}}{p}$.

    Taking $Q \sim \mathcal{B}(p)$ a Bernoulli of parameter $p$ independent from $Z^+$, one can readily notice that $Z \sim Q Z^+$. We thus   take $N$ i.i.d.~such copies $(Q_i, Z_i^+)_{i\in [N]}$. Let us take any $q \in \{0, 1\}^N$ so that $\abs{q} := \sum_{i\in[N]} q_i > 0$,

    \begin{align*}
        \esp{A | Q = q} &= \esp{\frac{\sum_{i\in [N]} q_i Z^+_i}{\abs{q}}}\\
        &= \frac{\sum_{i : q_i = 1} \esp{Z^+_i}}{\abs{q}}\\
        &= \esp{Z^+}.
    \end{align*}
    Given that $\esp{A | \sum_{i\in [N]} Q_i = 0} = 0$ and that $\prob{\sum_{i\in [N]} Q_i \neq 0} = 1 - (1 - p)^N$, we get \eqref{an:eq:esp_sum_z}.

    We will denote $\variance{A | Q = q} = \esp{A^2 | Q=q} - \esp{A|Q=q}^2$. We study
    \begin{align*}
        \esp{A^2 | Q = q} &= \variance{A | Q = q} + \esp{A | Q = q}^2\\
        &= \variance{\frac{\sum_{i : q_i=1} Z^+_i}{\abs{q}}} + \frac{\esp{Z}^2}{p^2} \\
        &= \frac{\variance{Z^+}}{\abs{q}} + \frac{\esp{Z}^2}{p^2}\\
        &= \frac{\esp{Z^2}}{p\abs{q}} - \frac{\esp{Z}^2}{p^2 \abs{q}} + \frac{\esp{Z}^2}{p^2}.
    \end{align*}
    \begin{align*}
        \esp{A^2}
        &= \sum_{k \in [N]} \esp{A^2 | \abs{Q} = k} \prob{\abs{Q} = k}\\
        &= \sum_{k \in [N]} \binom{N}{k} p^{k} (1-p)^{N-k} \frac{1}{k}\left(
            \frac{\esp{Z^2}}{p} - \frac{\esp{Z}^2}{p^2}\right)
            + \frac{\esp{Z}^2}{p^2}
                 (1 - (1 - p)^N)\\
        &\leq (1 - (1-p)^N)
            \frac{\esp{Z^2}}{p} + (1 - (1-p)^N) \frac{\esp{Z}^2}{p^2},
    \end{align*}
    as $\sum_{k \in [N]} \binom{N}{k} p^{k} (1-p)^{N-k} \frac{1}{k} \leq 1 - (1-p)^N$ which gives us~\eqref{an:eq:var_sum_z}
    and conclude this proof.
\end{proof}
Thanks to Lemma~\ref{an:lemma:esp_var} we get
\begin{align}
    &\forall k\in [d], \esp{\gab{n}^{(k)}} = \frac{1 - (1- p^{(k)})^B}{p^{(k)}} \esp{F'(w_{n-1})}\\
    \label{an:eq:var_ab}
    &\forall k\in [d], \esp{\left(\gab{n}^{(k)}\right)^2} \leq \frac{(1 - (1- p^{(k)})^N)}{p}
    \esp{\left(f'(w_{n-1})^{(k)}\right)^2} + \frac{(1 - (1- p^{(k)})^N)}{p^2} \norm{F'(w_{n-1})^{(k)}}^2.
\end{align}

We now present an improved bound for the second order moment of $Z$
that can be better if than the previous one in the case where $N p$ is large enough.
Although we will not provide a full proof of convergence using this result for simplicity,
we will comment on how this impact convergence in the proof of theorem~\ref{an:thm_convergence_ab}.

\begin{lemma}
\label{an:lemma:var_improved}
    With the same notation as in lemma~\ref{an:lemma:esp_var}, if $N p \geq 5$ we have
    \begin{align}
        \label{an:eq:var_sum_z_complex}
        \esp{A^2} &\leq\frac{5 (1 - (1-p)^N) \esp{Z^2}}{N p^2} + \frac{(1 - (1-p)^N) \esp{Z}^2}{p^2}.
    \end{align}
\end{lemma}
\begin{proof}
We reuse the notation from the proof of Lemma~\ref{an:lemma:esp_var}. Let us define $M := \abs{Q}$
which follows a binomial law of parameter $N$ and $p$. Using  Chernoff's inequality, we have
for any $k \leq N p$,
\begin{align*}
    \prob{M \leq k} &\leq \exp{\left(-\frac{(N p - k)^2}{2 N p}\right)},
\end{align*}
taking $k = \frac{Np}{2}$ we obtain
\begin{align*}
    \prob{M \leq \frac{N p}{2}} &\leq \exp{\left(-\frac{N p}{8}\right)}.
\end{align*}
We have
\begin{align*}
    \esp{\frac{1}{M} | M > 0}
    &\leq \prob{M \leq \frac{N p}{2} | M > 0}
     + \frac{2}{N p}\\
     &= \frac{\prob{M \leq \frac{N p}{2}}}{\prob{M > 0}}
     + \frac{2}{N p}\\
     &\leq \frac{\exp{\left(-\frac{N p}{8}\right)}}{1 - (1-p)^N} + \frac{2}{N p}.
\end{align*}
We have as $p \geq 5/N$ and using standard analysis techniques,
\begin{align*}
    \frac{\exp{\left(-\frac{N p}{8}\right)}}{1 - (1-p)^N}
    &\leq \frac{3}{N p}.
\end{align*}
We obtain
\begin{align*}
    \esp{\frac{1}{M} | M > 0} &\leq \frac{5}{N p}.
\end{align*}
Plugging this result into (\ref{an:eq:var_sum_z}), we immediately have
\begin{align*}
    \esp{A^2} &\leq \frac{5 (1 - (1-p)^N) \esp{Z^2}}{N p^2} + \frac{(1 - (1-p)^N) \esp{Z}^2}{p^2}.
\end{align*}
\end{proof}

\ansection{Proof of convergence of AdaBatch and mini-batch SGD}
\label{an:sec:proof_adabatch_sgd}

\ansubsection{Constant step size SGD with mini-batch}

We will first give a convergence result for the regular mini-batch SGD, which is adapted from \cite{proof_sgd}.

\begin{assumption}
\label{an:assumptions}
    We assume there exists a convex compact set $\mathcal{D}\subset \reel^d$ so that $f$ and $F$ verifies the following assumptions for $\mu$, $L$ and $R$ strictly positive,
    \begin{enumerate}
        \item The hessian $F''$ of $F$ is bounded from above and below as:
        \begin{equation}
        \label{an:eq:assumption:lipchitz}
            \forall w \in \mathcal{D}, \mu \Id \preceq F''(w) \preceq L \Id,
        \end{equation}
         with $\mu > 0$ so that $f$ is $\mu$ strongly convex
        and $L$ smooth over $\mathcal{D}$.
        \item We assume $f''$ is almost surely bounded,
        \begin{equation}
        \label{an:eq:assumption:as_liptchitz}
            \forall w \in \mathcal{D},
            f''(w) \preceq R^2 \Id.
        \end{equation}
        \item Let $w_* := \argmin_{w \in \mathcal{D}} F(w)$,
        \begin{align}
        \label{an:eq:global_minimizer}
            &F'(w_*) = 0,
        \end{align}
        which means in particular that $w_*$ is a global minimizer of $F$ over $\reel^d$.
    \end{enumerate}
\end{assumption}

This does not limit us to the case of \emph{globally} strongly convex functions $F$ as we only require it to be strongly convex on a compact subset that contains the global optimum $w_*$. In practice, this is often going to be the case, even when using a non strictly convex loss such as the logistic loss as soon as the problem is not perfectly separable, i.e., there is no hyperplane that perfectly separates the classes we are trying to predict.

We will now study the recursion for a given $w_0 \in \reel^d$ given by
\begin{align}
\label{an:eq:recursion_sgd}
    \forall n > 0, w_n = \Pi_{\mathcal{D}} \left[w_{n-1} - \frac{\gamma}{B}\sum_{b\in [B]} f'_{n, b}(w_{n-1})\right],
\end{align}
 where $\Pi_{\mathcal{D}}[w] := \argmin_{x\in \mathcal{D}}\norm{w - x}^2$ is the orthogonal projection on the set $\mathcal{D}$. This extra step of projection is required for this proof technique but experience shows that it is not needed.

\begin{thm}[Convergence of $F_n - F_*$ for SGD with mini-batch]
\label{an:thm_convergence_mini_batch}
If Assumptions~\ref{an:assumptions} are verified and
\begin{align}
\label{an:eq:condition_step_size}
    \gamma \left[L \left(1 - \frac{1}{B}\right) + \frac{2 R^2}{B}\right] \leq 1,
\end{align}
then for any $N > 0$,
    \begin{equation}
    \label{an:eq:convergence_norm_mb}
        \norm{w_N - w_*}^2 \leq (1 - \gamma \mu/2)^N \norm{w_0 - w_*}^2 + \frac{4\gamma}{B \mu} \esp{\norm{f'(w_*)}^2},
    \end{equation}
    and introducing \begin{equation*}
        \bar{w}_N = \frac{\sum_{n\in[N]} (1 - \gamma \mu/2)^{N-n}w_n}{
        \sum_{n\in[N]}(1 - \gamma \mu/2)^{N-n}},
    \end{equation*} we have
    \begin{equation}
    \label{an:eq:convergence_f_mb}
        \esp{F(\bar{w}_N)} - F_* \leq
            \gamma^{-1}(1 - \gamma \mu /2)^N \norm{w_0 - w_*}^2
            + \frac{2 \gamma  }{B} \esp{\norm{f'(w_*)}^2}.
    \end{equation}

\end{thm}
Introducing $\bar{w}_N$ allows for an easier comparison directly on the objective
function. This is made for qualitative analysis and we do not in practice perform this averaging.

We can see that the error given by \eqref{an:eq:convergence_f_mb} can be composed in two terms, one that measure how quickly we move away from the starting point and the second that depends on the stochastic noise around the optimum. We will call the former the \emph{bias} term and the latter the \emph{variance} term, following the terminology introduced by \cite{bach13}.

\begin{proof}
We introduce $\forall n \in [N], \mathcal{F}_{n-1}$ the $\sigma$-field generated by $(f_{i, b})_{i\in [n-1], b\in[B]}$. Let us take $n \in [N]$ and
introduce $\eta_n := w_n - w_*$ and $g_n := \frac{1}{B}\sum_{b\in[B]} f'_{n,b}(w_{n-1})$.
We then proceed to bound $\norm{\eta_n}^2$,
\begin{align*}
    \norm{\eta_n}^2 &\leq \norm{\eta_{n-1} - \gamma g_n}^2
    \quad \text{as $\Pi_{\mathcal{D}}$ is contractant for $\norm{\cdot}$}\\
    &=
        \norm{\eta_{n-1}}^2
        -2 \gamma g_n^T \eta_{n-1}
        + \gamma^2 \norm{g_n}^2.
\end{align*}

Taking the expectation while conditioning on $\mathcal{F}_{n-1}$ we obtain
\begin{align}
\label{an:eq:ref_cond_esp_eta}
    \esp{\norm{\eta_n}^2 | \mathcal{F}_{n-1}} &\leq
        \norm{\eta_{n-1}}^2
        -2 \gamma F'(w_{n-1})^T\eta_{n-1}
        + \gamma^2 \esp{\norm{g_n}^2 | \mathcal{F}_{n-1}},
\end{align}
\begin{align*}
    \esp{\norm{g_n}^2 | \mathcal{F}_{n-1}} &=
        \frac{\esp{\norm{f'(w_{n-1})}^2}}{B} +
        \norm{F'(w_{n-1})}^2 \left(1 - \frac{1}{B}\right).
\end{align*}
Injecting this in~\eqref{an:eq:ref_cond_esp_eta} gives us
\begin{align}
\label{an:eq:before_convex_mb}
    \esp{\norm{\eta_n}^2 | \mathcal{F}_{n-1}} &\leq
        \norm{\eta_{n-1}}^2
        -2 \gamma F'(w_{n-1})^T\eta_{n-1}
        + \frac{\gamma^2}{B} \esp{\norm{f'(w_{n-1})}^2}
        + \gamma^2 \norm{F'(w_{n-1})}^2 \left(1 - \frac{1}{B}\right).
\end{align}
As $F'' \preceq L \Id $ and using the co-coercivity of $F'$ we have
\begin{align*}
    \norm{F'(w_{n-1})}^2 &= \norm{F'(w_{n-1}) - F'(w_*)}^2 \\
        &\leq L  (F'(w_{n-1}) - F'(w_*))^T (w_{n-1} - w_*)\\
        &= L   F'(w_{n-1})^T (w_{n-1} - w_*).
\end{align*}

We have
\begin{align*}
    \norm{f'(w_{n-1})}^2 &\leq   2\norm{f'(w_{n-1}) - f'(w_*)}^2 +   2\norm{f'(w_{*})}^2
    \leq 2R^2   (f'(w_{n-1}) - f'(w_*))^T(w_{n-1} - w_*) +   2\norm{f'(w_{*})}^2
\end{align*}
and
\begin{align*}
    \esp{\norm{f'(w_{n-1})}^2  | \mathcal{F}_{n-1}} &\leq
    2  R^2 F'(w_{n-1})^T(w_{n-1} - w_*)
    +   2\norm{f'(w_{*})}^2.
\end{align*}

Injecting in~\eqref{an:eq:before_convex_mb} we get
\begin{align*}
    \esp{\norm{\eta_n}^2 | \mathcal{F}_{n-1}} &\leq
        \norm{\eta_{n-1}}^2
- \gamma F'(w_{n-1})^T\eta_{n-1} \underbrace{\left(
    2 - \gamma L\left(1 - \frac{1}{B}\right) - \frac{2\gamma R^2}{B}
\right)}_{A}
        + \frac{2\gamma^2\esp{\norm{f'(w_*)}^2}}{B}.
\end{align*}
We want $A$ to be large enough, we will take
\begin{align}
    \gamma \left[L \left(1 - \frac{1}{B}\right) + \frac{2 R^2}{B}\right] \leq 1,
\end{align}
which gives us $A \geq 1$ and
\begin{align*}
    \esp{\norm{\eta_n}^2 | \mathcal{F}_{n-1}} &\leq
        \norm{\eta_{n-1}}^2
- \gamma F'(w_{n-1})^T\eta_{n-1}
        + \frac{2 \gamma^2\esp{\norm{f'(w_*)}^2}}{B}.
\end{align*}
As $F$ is $\mu$ strictly convex, we have
\begin{align*}
    F_* - F(w_{n-1}) \geq F'(w_{n-1})^T(w_* - w_{n-1}) + \frac{\mu}{2} \norm{\eta_{n-1}}^2,
\end{align*}
which allows to obtain
\begin{align}
\label{an:eq:proof_mb_full_decrease}
    \esp{\norm{\eta_n}^2 | \mathcal{F}_{n-1}} &\leq
        (1 - \gamma \mu / 2) \norm{\eta_{n-1}}^2
        - \gamma (F(w_{n-1}) - F_*)
        + \frac{2\gamma^2}{B} \esp{\norm{f'(w_*)}^2}.
\end{align}
Taking the full expectation gives us
\begin{align*}
    \esp{\norm{\eta_n}^2} &\leq
        (1 - \gamma \mu / 2) \esp{\norm{\eta_{n-1}}^2}
        - \gamma (\esp{F(w_{n-1})} - F_*)
        + \frac{2\gamma^2\esp{\norm{f'(w_*)}^2}}{B},\\
        &\leq
        (1 - \gamma \mu / 2)^n \norm{\eta_0}^2
        + \frac{2\gamma^2\esp{\norm{f'(w_*)}^2}}{B} \sum_{0 \leq i < n} (1 - \gamma \mu / 2)^i\\
        &\leq (1 - \gamma \mu / 2)^n \norm{\eta_0}^2
        + \frac{4 \gamma}{B \mu} \esp{\norm{f'(w_*)}^2},
\end{align*}
which gives us~\eqref{an:eq:convergence_norm_mb}.

Let us now take $\alpha := (1 - \gamma \mu / 2)$, let us call $u_n := \esp{\norm{\eta_n}^2}$, we have using~\eqref{an:eq:proof_mb_full_decrease},
\begin{align*}
\gamma \delta_{n-1} &\leq \alpha u_{n-1} - u_n + 2 \gamma^2 R^2\\
\gamma \delta_{n-1} \alpha^{-n} &\leq \alpha^{-n + 1} u_{n-1} - u_n \alpha^{-n}
    + 2 \gamma^2 R^2 \alpha^{-n},
\end{align*}
summing for $n$ from $1$ to $N$ we obtain
\begin{align*}
    \gamma \sum_{n\in [N]} \delta_{n-1} \alpha^{-n}
    \leq u_0 - u_N \alpha^{-N}
        + \frac{2\gamma^2}{B} \esp{\norm{f'(w_*)}^2}\sum_{n\in[N]} \alpha^{-n},
\end{align*}
dividing by $\sum_{n\in[N]} \alpha^{-n}$ on each side and using the convexity of $F$ we get
\begin{align*}
    \esp{F(\bar{w}_{N-1})} - F_* \leq
        \alpha^N u_0 + \frac{2\gamma^2}{B} \esp{\norm{f'(w_*)}^2},
\end{align*}
which gives us~\eqref{an:eq:convergence_f_mb} and concludes this proof.
\end{proof}

\ansubsection{Convergence of reconditioned SGD}
\label{an:subsection:reconditioned_sgd}
Let us now assume that we have for some matrices $T$ and $C$ definite positive so that
\begin{align*}
    \forall w\in\mathcal{D}, \mu T\preceq F''(w) \preceq L T,\\
    \forall w\in\mathcal{D}, f''(w) \preceq L \Id \quad  a.s.
\end{align*}
We now study $w_n$ defined by the following recurence
\begin{align}
\label{an:recondition:recurrence_rule}
    w_n = w_{n-1} - \frac{\gamma}{B} C \sum_{b\in[B]} f'_{n, b}(w_{n-1}).
\end{align}
First let us introduce $v_0 := \sqrt{C}^{-1} w_0$, $v_* := \sqrt{C}^{-1} w_*$ and $h(w) := f(\sqrt{C}w)$ as well
as $\forall n \in[N], b \in [B], h_{n,b}(w) := f_{n,b}(\sqrt{C}w)$, then we define
\begin{align*}
    v_n &:= v_{n-1} - \frac{\gamma}{B} \sum_{b\in[B]} h'_{m,b}(v_{n-1}).
\end{align*}
Multiplying by $\sqrt{C}$ we recover the same recurrence rule as~\eqref{an:recondition:recurrence_rule} for $w_n = \sqrt{C} v_n$. Therefore, the convergence of $v_n$ will give us the convergence of $w_n$.

Let us take $H(w) := \esp{h(w)} = F(\sqrt{C}w)$. By definition we have
\begin{align*}
    \forall w\in\mathcal{D}_C, \mu \sqrt{C}T\sqrt{C} \preceq F''(w) \preceq L \sqrt{C}T\sqrt{C},\\
    \forall w\in\mathcal{D}_C, h''(w) \preceq R^2 L_C\Id \quad  a.s.,
\end{align*}
where $\mathcal{D}_C = \sqrt{C}^{-1} \mathcal{D}$, $L_C$ is the largest eigen value of $C$.
 If we take $\mu_{C, T}$ (resp $L_{C,T}$) the smallest (resp largest) eigenvalue
 of $\sqrt{C}T\sqrt{C}$, then using Theorem~\ref{an:thm_convergence_mini_batch}, we have
 for
 \begin{align}
 \label{an:bound_gamma_reco}
    \gamma \left[L L_{C, T} \left(1 - \frac{1}{B}\right) + \frac{2 L_C R^2}{B}\right] \leq 1,
\end{align}
\begin{equation*}
    \norm{v_N - v_*}^2 \leq (1 - \gamma \mu_{C,T})^N \norm{v_0 - v_*}^2 + \frac{2\gamma}{B \mu} \esp{\norm{h'(v_*)}^2},
\end{equation*}
and introducing
\begin{equation*}
    \bar{v}_N = \frac{\sum_{n\in[N]} (1 - \gamma \mu_{C,T}/2)^{N-n}v_n}{
    \sum_{n\in[N]}(1 - \gamma \mu_{C,T}/2)^{N-n}},
\end{equation*} we have
\begin{equation*}
    \esp{H(\bar{v}_N)} - H_* \leq
        \gamma^{-1}(1 - \gamma \mu_{C,T})^N \norm{v_0 - v_*}^2
        + \frac{2 \gamma L_C R^2}{B} \esp{\norm{h'(v_*)}^2}.
\end{equation*}
Using $w_n = \sqrt{C} v_n$, we obtain
\begin{equation*}
    \norm{w_N - w_*}_{C^{-1}}^2 \leq (1 - \gamma \mu_{C,T})^N \norm{w_0 - w_*}^2_{C^{-1}} + \frac{4\gamma}{B \mu} \esp{\norm{f'(w_*)}^2_C},
\end{equation*}
\begin{equation}
 \label{an:conv_reco}
    \esp{F(\bar{w}_N)} - F_* \leq
        \gamma^{-1}(1 - \gamma \mu_{C,T})^N \norm{w_0 - w_*}_{C^{-1}}^2
        + \frac{2 \gamma L_C}{B} \esp{\norm{f'(w_*)}^2_C}.
\end{equation}

\paragraph{Application to sparse optimization.}
In the sparse setting, we have made the assumption that $T := \Diag{p}$.
We suggested two reconditioning strategies in such case. The first one is to take
$C = \Diag{p}^{-1}$. In such case $\mu_{C,T} = L_{C, T} = 1$ so that we have a perfect conditioning. However $L_C = \pmin^{-1}$ so that if $B \ll \pmin^{-1}$ we would have to take a much smaller step size and the term $\frac{1}{B}\esp{\norm{f'(w_*)}^2}_{\Diag{p}^{-1}}$ would explode.

The second one, $C := \Diag{\frac{1 - (1- p)^B}{p}}$ so that
$\mu_{C,T} = 1 - (1- \pmin)^B$ and $L_{C, T} = 1 - (1- \pmax)^B$.
Because the fonction $p \rightarrow 1 - (1- p)^B$ increases faster for small probabilities, the conditioning of the problem is improved. If $\pmax$ is close to 1 and $\pmin$ close to 0, then $\mu_{C,T} \approx B \pmin$ and $L_C \approx \pmax$.
We have $L_C \leq B$ as $\forall p\in [0, 1], (1 - (1-p)^B) \leq B p$, so that the increase due to $L_C$ is perfectly
balanced out by the batch size $B$ in \eqref{an:bound_gamma_reco}.
Besides, as $\frac{C}{B} \preceq \Id$, we have $\frac{1}{B}\esp{\norm{f'(w_*)}_{C}^2} \leq \esp{\norm{f'(w_*)}^2}$.

\ansubsection{Convergence of AdaBatch}
We now make the following assumptions:
\begin{assumption}
\label{an:assumptions_ab}
    We assume there exists a convex compact set $\mathcal{D}\subset \reel^d$, $\mu$, $L$ and $R$ strictly positive so that we have the following assumptions verified.
    \begin{enumerate}
        \item The Hessian $F''$ (resp $f''$) of $F$ (resp $f$) are bounded from above and below as:
        \begin{equation}
            \forall w \in \mathcal{D}, \ \  \mu \Diag{p} \preceq F''(w) \preceq L \Diag{p}
            \quad \text{and}\quad
            \forall w \in \mathcal{D},
            f''(w) \preceq R^2 \Id.
        \end{equation}
        \item Let $w_* := \argmin_{w \in \mathcal{D}} F(w)$,
        \begin{align}
        \label{an:eq:global_minimizer_ab}
            &F'(w_*) = 0.
        \end{align}
        In particular, $w_*$ is a global minimizer of $F$ over $\reel^d$.
    \end{enumerate}
\end{assumption}

We will now study convergence when we use the AdaBatch update. We are given $w_0 \in \reel^d$
and we define $C:=\Diag{\frac{1-(1-p)^B}{p}}$ and we define recursively,
\begin{align*}
    &\forall k \in [d], \gab{n}^k =
        \begin{cases}
            \frac{\sum_{b \in B^k_n} f'_{n,b}(w_{n-1})^k}{
                \sum_{b : k\in \Suppo{f}} 1} \quad &\text{if $\sum_{b : k\in \Suppo{f}} 1 \neq 0$}\\
            0 \quad &\text{otherwise}
        \end{cases}\\
    &w_n = \Pi_{\mathcal{D}, C}\left[w_{n-1} - \gamma \gab{n}\right],
\end{align*}
 where $\Pi_{\mathcal{D}, C}[w] := \argmin_{x\in \mathcal{D}}\norm{w - x}_{C^{-1}}^2$ is the projection on the set $\mathcal{D}$ with respect to $\norm{\cdot}_{C^{-1}}$. This extra step of projection is required for this proof technique but experience shows that it is not needed.
 We introduce $\forall p \in [0,1], p^{+B} := 1-(1-p)^B$.

\begin{thm}[Convergence of $F_n - F_*$ for AdaBatch]
\label{an:thm_convergence_ab}
If Assumptions~\ref{an:assumptions_ab} are verified and
\begin{align}
\label{an:eq:condition_step_size_ab}
    \gamma \left(L + 2 R^2\right) \leq 1,
\end{align}
then for any $N > 0$,
    \begin{equation}
    \label{an:eq:convergence_norm_ab}
        \norm{w_N - w_*}^2 \leq (1 - \gamma\pmin^{+B} \mu/2)^N \norm{w_0 - w_*}^2 + \frac{4\gamma}{\mu} \esp{\norm{f'(w_*)}^2},
    \end{equation}
    and introducing \begin{equation*}
        \bar{w}_N = \frac{\sum_{n\in[N]} (1 - \gamma \pmin^{+B}\mu/2)^{N-n}w_n}{
        \sum_{n\in[N]}(1 - \gamma \pmin^{+B}\mu/2)^{N-n}},
    \end{equation*} we have
    \begin{equation}
    \label{an:eq:convergence_f_ab}
        \esp{F(\bar{w}_N)} - F_* \leq
            \gamma^{-1}(1 - \gamma \pmin^{+B}\mu /2)^N \norm{w_0 - w_*}^2
            + 2 \gamma \esp{\norm{f'(w_*)}^2}.
    \end{equation}

\end{thm}
\begin{proof}
We introduce $\forall n \in [N], \mathcal{F}_{n-1}$ the $\sigma$-field generated by $(f_{i, b})_{i\in [n-1], b\in[B]}$. Let us take $n \in [N]$ and introduce
$\eta_n := w_n - w_*$.
We then proceed to bound $\norm{\eta_n}^2_{C^{-1}}$,
\begin{align*}
    \norm{\eta_n}_{C^{-1}}^2 &\leq \norm{\eta_{n-1} - \gamma g_n}_{C^{-1}}^2
    \quad \text{as $\Pi_{\mathcal{D}}$ is contractant for $\norm{\cdot}_{C^{-1}}$}\\
    &=
        \norm{\eta_{n-1}}_{C^{-1}}^2
        -2 \gamma g_n^T \eta_{n-1}
        + \gamma^2 \norm{g_n}_{C^{-1}}^2.
\end{align*}

Taking the expectation while conditioning on $\mathcal{F}_{n-1}$ we obtain
\begin{align}
\label{an:eq:ref_cond_esp_eta_ab}
    \esp{\norm{\eta_n}^2_{C^{-1}} | \mathcal{F}_{n-1}} &\leq
        \norm{\eta_{n-1}}^2_{C^{-1}}
        -2 \gamma F'(w_{n-1})^T\eta_{n-1}
        + \gamma^2 \esp{\norm{g_n}^2_{C^{-1}} | \mathcal{F}_{n-1}}.
\end{align}
Using Lemma~\ref{an:lemma:esp_var},
\begin{align}
\label{an:eq:easy_bound_on_ab_gradient}
    \esp{\norm{g_n}^2_{C^{-1}} | \mathcal{F}_{n-1}}
        &\leq \esp{\norm{f'(w_{n-1})}^2} + \norm{F'(w_{n-1})}_{\diag{p}^{-1}}^2.
\end{align}
Although we will not do it in the following, it is also possible to use Lemma~\ref{an:lemma:var_improved}.
Indeed, for any $k \in [d]$ such that $B p^{(k)} \geq 5$, we have
\begin{align*}
    \esp{\left(g_n^{(k)}\right)^2 C^{-1}_{k,k} | \mathcal{F}_{n-1}}
        &\leq \frac{5 \esp{(f'(w_{n-1})^{(k)})^2}}{N p} + \frac{(F'(w_{n-1})^{(k)})^2}{p^{(k)}},
\end{align*}
which would be equivalent for the dimension $k$ to having a regular batch size of $\frac{p^{(k)}B}{5}$.
This shows that AdaBatch will benefit from a reduced variance for features that are frequent enough.
For simplicity we will however stick with the simpler bound given by~\eqref{an:eq:easy_bound_on_ab_gradient}.

As $F'' \preceq L \diag{p} $ and using the co-coercivity of $F'$ we have
\begin{align*}
    \norm{F'(w_{n-1})}_{\diag{p}^{-1}}^2 &= \norm{F'(w_{n-1}) - F'(w_*)}_{\diag{p}^{-1}}^2 \\
        &\leq L  (F'(w_{n-1}) - F'(w_*))^T (w_{n-1} - w_*)\\
        &= L   F'(w_{n-1})^T (w_{n-1} - w_*).
\end{align*}
Injecting this in~\eqref{an:eq:ref_cond_esp_eta_ab} gives us
\begin{align}
\label{an:eq:before_convex_ab}
    \esp{\norm{\eta_n}^2 | \mathcal{F}_{n-1}} &\leq
        \norm{\eta_{n-1}}_{C^{-1}}^2
        -2 \gamma F'(w_{n-1})^T\eta_{n-1}
        + \gamma^2 \esp{\norm{f'(w_{n-1})}^2}.
\end{align}
We have
\begin{align*}
    \norm{f'(w_{n-1})}^2 &\leq   2\norm{f'(w_{n-1}) - f'(w_*)}^2 +   2\norm{f'(w_{*})}^2
    \leq 2R^2   (f'(w_{n-1}) - f'(w_*))^T(w_{n-1} - w_*) +   2\norm{f'(w_{*})}^2
\end{align*}
and
\begin{align*}
    \esp{\norm{f'(w_{n-1})}^2 | \mathcal{F}_{n-1}} &\leq
    2  R^2 F'(w_{n-1})^T(w_{n-1} - w_*)
    +   2\norm{f'(w_{*})}^2 .
\end{align*}
Injecting in~\eqref{an:eq:before_convex_ab} we get
\begin{align*}
    \esp{\norm{\eta_n}_{C^{-1}}^2 | \mathcal{F}_{n-1}} &\leq
        \norm{\eta_{n-1}}_{C^{-1}}^2
- \gamma F'(w_{n-1})^T\eta_{n-1} \underbrace{\left(
    2 - \gamma L - 2\gamma R^2
\right)}_{A}
        + 2\gamma^2\esp{\norm{f'(w_*)}^2}.
\end{align*}
We want $A$ to be large enough, we will take
\begin{align}
    \gamma \left(L + 2 R^2 \right) \leq 1,
\end{align}
which gives us $A \geq 1$ and
\begin{align*}
    \esp{\norm{\eta_n}^2_{C^{-1}} | \mathcal{F}_{n-1}} &\leq
        \norm{\eta_{n-1}}_{C^{-1}}^2
- \gamma F'(w_{n-1})^T\eta_{n-1}
        + 2 \gamma^2\esp{\norm{f'(w_*)}^2}.
\end{align*}
As $F'' \succeq \mu \Diag{p}$ we have
\begin{align*}
    F_* - F(w_{n-1}) &\geq F'(w_{n-1})^T(w_* - w_{n-1}) + \frac{\mu}{2} \norm{\eta_{n-1}}_{\Diag{p}}^2
    \\&\geq F'(w_{n-1})^T(w_* - w_{n-1}) + \frac{\pmin^{+B}\mu}{2} \norm{\eta_{n-1}}_{C^{-1}}^2,
\end{align*}
which allows to obtain
\begin{align}
    \esp{\norm{\eta_n}^2_{C^{-1}} | \mathcal{F}_{n-1}} &\leq
        (1 - \gamma \mu \pmin^{+B} / 2) \norm{\eta_{n-1}}_{C^{-1}}^2
        - \gamma (F(w_{n-1}) - F_*)
        + 2\gamma^2\esp{\norm{f'(w_*)}^2}.
\end{align}
Taking the full expectation gives us
\begin{align*}
    \esp{\norm{\eta_n}^2_{C^{-1}}} &\leq
        (1 - \gamma \mu \pmin^{+B}/ 2) \esp{\norm{\eta_{n-1}}_{C^{-1}}^2}
        - \gamma (\esp{F(w_{n-1})} - F_*)
        + 2\gamma^2\esp{\norm{f'(w_*)}^2},\\
        &\leq
        (1 - \gamma \mu \pmin^{+B}/ 2)^n \norm{\eta_0}_{C^{-1}}^2
        + 2\gamma^2\esp{\norm{f'(w_*)}^2} \sum_{0 \leq i < n} (1 - \gamma \mu \pmin^{+B}/ 2)^i\\
        &\leq (1 - \gamma \mu \pmin^{+B}/ 2)^n \norm{\eta_0}_{C^{-1}}^2
        + \frac{4 \gamma}{\mu} \esp{\norm{f'(w_*)}^2},
\end{align*}
which gives us~\eqref{an:eq:convergence_norm_ab}.
We obtain~\eqref{an:eq:convergence_f_ab} in the exact same way as in the proof of Theorem~\ref{an:thm_convergence_mini_batch}.
\end{proof}

\ansubsection{Sparse linear prediction}
\label{an:sec:sparse_linear_prediction}
We will now show that Assumption~\ref{an:assumptions_ab} is easy to meet in the case of linear predictions.
For simplicity, let us assume $X$ is a random variable with values in $\{0,1\}^d$
with uncorrelated features, i.e.,
\begin{align*}
\forall k, k' \in [d]: k \neq k', \esp{X^{(k)} X^{(k')}} &= \esp{X^{(k)}}{X^{(k')}},
\end{align*} and $\phi : \reel \rightarrow \reel$ a random
convex function.
Then one can take $f(w) := \phi(X^T w)$.
For $m$, $M$ and $G$ strictly positive and $\mathcal{D}\subset \reel^d$ a convex compact, We assume almost surely we have
\begin{align*}
    &\forall w \in \mathcal{D}, m \leq \phi''(w) \leq M,\\
    &\norm{X}^2 \leq G^2 \quad a.s.
\end{align*}
Then, for any $w \in \mathcal{D}$ we have
\begin{align*}
    F''(w) &= \esp{f''(w)}\\
        &= \esp{\phi''(w) X X^T}\\
        &\preceq M \esp{X X^T}\\
        &= M \left(\Diag{p} - \Diag{p}^2 + pp^T\right).
\end{align*}
Moreover, we have
\begin{align}
    pp^T &= \sqrt{\Diag{p}} \left(
        \sqrt{p}\sqrt{p}^T
    \right)\sqrt{\Diag{p}}.
\end{align}
As $\sqrt{p}\sqrt{p}^T \preceq \norm{\sqrt{p}}^2 \Id = \sum_{k\in[d]} p^{(k)} \Id$, we obtain
\begin{align*}
    F''(w) \preceq M \left(1 + \sum_{k\in[d]} p^{(k)}\right) \Diag{p}.
\end{align*}
Besides, we have
\begin{align*}
    F''(w) &\succeq m  \left(\Diag{p} - \Diag{p}^2\right)\\
    &\succeq m (1 - \pmax) \Diag{p}.
\end{align*}
Finally,
\begin{align*}
    f''(w) &= \phi''(w) X X^T\\
    &\preceq M G^2.
\end{align*}

As a conclusion, Assumptions~\ref{an:assumptions_ab} are verified for $\mu:= m (1- \pmax)$, $L:= M \left(1 + \sum_{k\in[d]} p^{(k)}\right)$ and $R^2 := G^2 M$.

\ansection{AdaBatch for SVRG}
\label{an:sec:proof_adabatch_svrg}

\ansubsection{Mini-batch SVRG}
We now only assume that $F$ verifies the following inequalities for $\mu > 0$,
\begin{equation}
\label{an:eq:hessian_svrg}
    \forall w \in \reel^d, \ \  \mu\preceq F''(w) \quad \text{and} \quad  f(w) \preceq L \;\text{a.s}.
\end{equation}

Let us take
a starting point $y_0 \in \reel^d$ and $m \in \mathbb{N}^*$. For all $s = 0, 1, \ldots$ we have
$w_{s, 0} := y_s$ and for all $n \in [m]$ let us define
\begin{align*}
    w_{s, n} &:= w_{s, n-1} - \gamma g_{s,n},\\
    y_{s + 1} &:= \frac{1}{m} \sum_{n\in[m]} w_{s, n}.
\end{align*}
with $g_{s,n}$ the SVRG update based on $(f_{s,n,b})_{b\in[B]}$ i.i.d samples of $f$.
Let us introduce
\begin{equation*}
\forall k \in [d], D^{(k)}_{s, n} := \left\{b \in [B]: k \in \Suppo{f_{s,n,b}'}\right\}.
\end{equation*}
For any dimension $k \in [d]$ we have
\begin{align*}
g_{s,n}^{(k)} := \frac{1}{B}\left(
\sum\limits_{b \in D^{(k)}_{s,n}} f_{s,n,b}'(w_{s, n-1})^{(k)} - f_{s,n,b}'(y_s)^{(k)}
+ F'(y_s)^{(k)} / p^{(k)}\right).
\end{align*}

\begin{thm}[Convergence of SVRG with mini-batch]
\label{an:thm_convergence_svrg_mb}
If Assumptions~\ref{an:eq:hessian_svrg} are verified and $\gamma$ verifies
\begin{align*}
    \gamma L \left(1 - \frac{1}{B}\right) < 1,
\end{align*}
then for all $s > 0$ we have
\begin{equation}
\esp{F(y_s) - F_*} \leq \alpha^s (F(y_0) - F_*)
\end{equation}
where
\begin{equation}
\label{an:eq:svrg_alpha}
\alpha :=
            \frac{1}{\mu \gamma \left(1 - \frac{\gamma L(3 + B)}{2 B}\right) m}
            + \frac{2 L \gamma}{B \left(1 - \gamma L \frac{(3 + B)}{2 B}\right)}.
\end{equation}
\end{thm}
\begin{proof}
We will reuse the proof technique from \cite[section 6.3]{bubeck2015convex}.
We introduce $\forall n \in [N], \mathcal{F}_{s, n-1}$ the $\sigma$-field generated by $(f_{u, i, b})_{u \in [s], i\in [n-1], b\in[B]}$.
For simplicity, we will drop all the $s$ indices.
We define $\Gamma_{n,b} := \diag{\left(\mathbbm{1}_{k \in \Suppo{f_{n,b}}} / p^{(k)}\right)_{k\in [d]}}$ and $\Gamma := \diag{\left(\mathbbm{1}_{k \in \Suppo{f}} / p^{(k)}\right)_{k\in [d]}}$ so that
\begin{align*}
    g_{n} = \frac{1}{B}\left(\sum_{b\in [B]} f_{n,b}'(w_{n-1}) - f_{n,b}'(y)
+ \Gamma_{n, b} F'(y_s) \right).
\end{align*}

One can immediately notice that
\begin{align*}
    \esp{g_n | \mathcal{F}_{n-1}}
    = F'(w_{n-1}).
\end{align*}
Besides, we have
\begin{align*}
    \esp{\norm{f'(w_{n-1}) - f'(y) + \Gamma F'(y)}^2 | \mathcal{F}_{n-1}} &\leq
    2 \esp{\norm{f'(w_{n-1}) - f'(w_*)}^2 +
    \norm{f'(y) - f'(w_*) + \Gamma F'(y)}^2 | \mathcal{F}_{n-1}}.
\end{align*}
We reuse the same proof as in \cite[Lemma 10]{hogwild_mania}.
Using the fact that $\esp{(f'(y) - f'(w_*))^T \Gamma F'(y)|\mathcal{F}_{n-1}} = \norm{F'(y)}^2_{\diag{p}^{-1}}$
and $\esp{\norm{\Gamma F'(y)}^2|\mathcal{F}_{n-1}} = \norm{F'(y)}^2_{\diag{p}^{-1}}$ we have
\begin{align*}
    \esp{A^{(k)} | \mathcal{F}_{n-1}} &=
    \esp{\norm{f'(y) - f'(w_*) }^2 | \mathcal{F}_{n-1}}
    - 2 \norm{F'(y)}^2_{\diag{p}^{-1}} + \norm{F'(y)}^2_{\diag{p}^{-1}}\\
    &\leq \esp{\norm{f'(y) - f'(w_*) }^2 | \mathcal{F}_{n-1}}.
\end{align*}
It follows that
\begin{align*}
    \esp{\norm{g_{n}}^2 | \mathcal{F}_{n-1}} &=
        \frac{1}{B} \esp{\norm{f'(w_{n-1}) - f'(y) + \Gamma F'(y)}^2 | \mathcal{F}_{n-1}}
        + \norm{F'(w_{t-1})}^2 \left(1 - \frac{1}{B}\right) \\
    &\leq \frac{2}{B}\left(
        \esp{\norm{f'(w_{n-1}) - f'(w_*)}^2 + \norm{f'(y) - f'(w_*)}^2 | \mathcal{F}_{n-1}}
        \right) + \norm{F'(w_{t-1})}^2 \left(1 - \frac{1}{B}\right)\\
    &\leq \frac{4 L}{B} (F(w_{n-1}) - F_* + F(y) - F_*) + \norm{F'(w_{t-1})}^2 \left(1 - \frac{1}{B}\right),
\end{align*}
using Lemma 6.4 from \cite{bubeck2015convex}. We also have
\begin{align*}
    \norm{F'(w_{t-1})}^2 &\leq L   F'(w_{n-1})^T (w_{n-1} - w_*),
\end{align*}
so that
\begin{align*}
    \esp{\norm{w_{n} - w_*}^2 | \mathcal{F}_{n-1}}
    \leq &\norm{w_{n-1} - w_*}^2 - 2 \gamma \left(1 - \frac{\gamma L}{2} \left(1 - \frac{1}{B}\right) \right) F'(w_{n-1})^T (w_{n-1} - w_*) \\
        &+ \frac{4 \gamma^2 L}{B} (F(w_{n-1}) - F_* + F(y) - F_*).
\end{align*}
We choose $\gamma$ so that
\begin{align*}
    \gamma L \left(1 - \frac{1}{B}\right) < 1,
\end{align*}
and using that $F'(w_{n-1})^T (w_{n-1} - w_*) \geq F(w_{n-1}) - F_*$ we have
\begin{align*}
    \esp{\norm{w_{n} - w_*}^2 | \mathcal{F}_{n-1}}
    \leq &\norm{w_{n-1} - w_*}^2 - \gamma \left(2 - \frac{\gamma L(3 + B)}{B} \right) \left(F(w_{n-1}) - F_*\right)
        + \frac{4 \gamma^2 L}{B} \left(F(y) - F_*\right).
\end{align*}
Summing the above inequality for $n \in [m]$ and taking the expectation with respect to $\mathcal{F}_0$,
\begin{align*}
    \esp{\norm{w_m - w_*}^2} \leq \norm{y - w_*} - \gamma \left(2 - \frac{\gamma L(3 + B)}{B} \right)
        \esp{\sum_{n \in [m]} F(w_n) - F_*| \mathcal{F}_0} + \frac{4 L \gamma^2 m}{B} (F(y) - F_*).
\end{align*}
Using the strong convexity of $F$ we have $\norm{w_0 - w_*}^2 \leq \frac{2}{\mu} \left(F(y) - F_*\right)$
and finally
\begin{align*}
    \esp{F\left(\frac{1}{m}\sum_{n \in [m]} w_n\right)
        - F_* | \mathcal{F}_0} \leq
            \left(
                \frac{1}{\mu \gamma \left(1 - \frac{\gamma L(3 + B)}{2 B}\right) m}
                + \frac{2 L \gamma}{B \left(1 - \gamma L \frac{(3 + B)}{2 B}\right)}
            \right) (F(y) - F(w_*)).
\end{align*}
\end{proof}

One can derive a simplified convergence result when we assume $B$ large enough.
\begin{corollary}
\label{an:corollary_svrg_mb}
If we assume $B \gg 1$, then with $\gamma = \frac{1}{L}$ and
$m = \frac{2 B L}{\mu (0.9 B - 4)} \approx \frac{2.2 L }{\mu}$,
we have
\begin{align*}
    \esp{F(y_s) - F_*} \leq 0.9^s (F(y_0) - F_*).
    \end{align*}
\end{corollary}

\ansubsection{AdaBatch SVRG}
Let us now assume that $F$ verify the following inequalities for $\mu > 0$,
\begin{equation}
\label{an:eq:hessian_svrg_adabatch}
    \forall w \in \reel^d, \ \  \mu \diag{p}\preceq F''(w) \quad \text{and} \quad  f(w) \preceq L \diag{p}\;\text{a.s}.
\end{equation}

We now define for any dimension $k$ such that $D^{(k)}_{s, n}\neq \emptyset$,
\begin{align*}
g_{s,n}^{(k)} := \frac{1}{\abs{D^{(k)}_{s,n}}}\left(
\sum\limits_{b \in D^{(k)}_{s,n}} f_{s,n,b}'(w_{s, n-1})^{(k)} - f_{s,n,b}'(y_s)^{(k)}
+ F'(y_s)^{(k)} / p^{(k)}\right),
\end{align*}
and $g_{s,n}^{(k)} := 0 $ otherwise.

\begin{thm}[Convergence of SVRG with AdaBatch]
\label{an:thm_convergence_svrg_ab}
If Assumptions~\ref{an:eq:hessian_svrg_adabatch} are verified and $\gamma$ verifies
\begin{align*}
    \gamma < \frac{L}{2},
\end{align*}
then for all $s > 0$ we have
\begin{equation}
\esp{F(y_s) - F_*} \leq \alpha^s (F(y_0) - F_*),
\end{equation}
where
\begin{equation}
\label{an:eq:svrg_alpha_ab}
\alpha :=
                \frac{1}{\mu (1 - (1-\pmin)^B) \gamma \left(1 - 2 \gamma L\right) m}
                + \frac{2 L \gamma}{1 - 2 \gamma L}.
\end{equation}
\end{thm}
\begin{proof}
We reuse the same proof technique as previously and introduce the same operators $\Gamma$ and $\Gamma_{n,b}$, again dropping all $s$ indices for simplicity.

We introduce $C:=\Diag{\frac{1-(1-p)^B}{p}}$ and using Lemma~\ref{an:lemma:esp_var} we have
\begin{align*}
    \esp{\norm{g_{n}}_{C^{-1}}^2 | \mathcal{F}_{n-1}}
    &\leq \esp{\norm{f'(w_{n-1}) - f'(y) + \Gamma F'(y)}^2 | \mathcal{F}_{n-1}}\\
    &\leq 4 L (F(w_{n-1}) - F_* + F(y) - F_*),
\end{align*}
using similar arguments as for regular mini-batch.
Therefore, we have
\begin{align*}
    \esp{\norm{w_{n} - w_*}_{C^{-1}}^2 | \mathcal{F}_{n-1}}
    \leq &\norm{w_{n-1} - w_*}_{C^{-1}}^2 - 2 \gamma F'(w_{n-1})^T (w_{n-1} - w_*)
        + 4 \gamma^2 L (F(w_{n-1}) - F_* + F(y) - F_*)\\
    \leq &\norm{w_{n-1} - w_*}_{C^{-1}}^2 - 2 \gamma \left(1 - 2 \gamma L \right) \left(F(w_{n-1}) - F_*\right)
        + 4 \gamma^2 L \left(F(y) - F_*\right).
\end{align*}
Summing the above inequality for $n \in [m]$ and taking the expectation with respect to $\mathcal{F}_0$,
\begin{align*}
    \esp{\norm{w_m - w_*}^2_{C^{-1}}} \leq \norm{y - w_*}_{C^{-1}} - 2 \gamma \left(1 - 2 \gamma L \right)
        \esp{\sum_{n \in [m]} F(w_n) - F_*| \mathcal{F}_0} + 4 L \gamma^2 m (F(y) - F_*).
\end{align*}
Using the strong convexity of $F$ we have
\begin{align*}
\norm{w_0 - w_*}_{C^{-1}} &\leq \frac{1}{1 - (1-\pmin)^B}\norm{w_0 - w_*}_{\diag{p}}^2 \\
&\leq \frac{2}{\mu (1 - (1-\pmin)^B)} \left(F(y) - F_*\right),
\end{align*}
and finally
\begin{align*}
    \esp{F\left(\frac{1}{m}\sum_{n \in [m]} w_n\right)
        - F_* | \mathcal{F}_0} \leq
            \left(
                \frac{1}{\mu (1 - (1-\pmin)^B)\gamma \left(1 - 2 \gamma L\right) m}
                + \frac{2 L \gamma}{1 - 2 \gamma L}
            \right) (F(y) - F(w_*)).
\end{align*}
\end{proof}

One can derive a simplified convergence result when we assume $\pmin$ small enough.
\begin{corollary}
\label{an:corollary_svrg_ab}
If we assume $\pmin \ll 1$ so that $(1 - (1 - \pmin)^B) \approx B \pmin$ then with $\gamma = \frac{1}{10 L}$ and
$m = \frac{20 L}{B \pmin \mu}$,
we have
\begin{align*}
    \esp{F(y_s) - F_*} \leq 0.9^s (F(y_0) - F_*).
\end{align*}
\end{corollary}

\ansubsection{Comparing the effect of regular mini-batch and AdaBatch for SVRG}

If $F$ verifies our sparse convexity condition
\begin{equation}
\label{an:eq:hessian_svrg_adabatch_2}
    \forall w \in \reel^d, \ \  \mu \diag{p}\preceq F''(w) \preceq L \diag{p},
\end{equation}
we can apply Theorem~\ref{an:thm_convergence_svrg_mb} for
\begin{equation}
\label{an:eq:hessian_svrg_mb_2}
    \forall w \in \reel^d, \ \  \mu \pmin\preceq F''(w) \preceq L,
\end{equation}

We will assume that the batch size $B$ is large enough (for instance $B=50$), $\pmin$ is small enough so that $1 - (1-\pmin)^B \approx B \pmin$.
Then using corollary~\ref{an:corollary_svrg_mb}, in order to achieve a linear convergence rate of 0.9 for regular mini-batch
we would need to have a number of inner iterations given by
\[
    m_{\textrm{mb}} \approx \frac{2.2 L}{\pmin \mu}.
\]
This number is roughly constant with the batch size. However the cost of each single iteration
is now $B$ times larger, thus meaning that we would need to process $B$ times more samples
before reaching the same accuracy as when $B = 1$.

On the other hand, using Corollary~\ref{an:corollary_svrg_ab}, in order to achieve the same
rate of convergence, we would require the number of inner iterations to be
\[
m_{\textrm{ab}} \approx \frac{20 L}{B \pmin \mu},
\]
thus the number of inner iterations is inversely proportionnal to the batch size, which
balances perfectly the increased cost of each iteration. We will reach the same accuracy as for $B = 1$
without requiring to process more samples.

It should be noted that using Lemma~\ref{an:lemma:var_improved}, it is possible to show that
AdaBatch also benefits from variance reduction for the coordinates where $p^{(k)} B$ is large enough.
This will depend on the datasets but we have observed such an effect in practice, which allows us
to take a larger step-size and further improve convergence.

As for regular SGD, we have noticed experimentally that mini-batch SVRG will become more efficient than AdaBatch when we are close to the optimum.
For instance, on datasets that are
much smaller than \emph{url} such as \emph{news20} or \emph{spam}, we observed that mini-batch SVRG will perform better than AdaBatch. Therefore, we would advice using AdaBatch for early optimization and regular mini-batch for fine tuning when close to the optimum.


\ansection{Experimental results}
\label{an:sec:experimental_results}
\ansubsection{Experimental results for AdaBatch Wild}
We present here the same graphs as in the main paper but for the \emph{spam} and \emph{url} dataset. We also provide the convergence with respect to the number of samples for \emph{news20}.
On both datasets, AdaBatch performs competitively with Hogwild! and significantly better than
mini-batch SGD, especially when increasing the number of workers and batch-size.

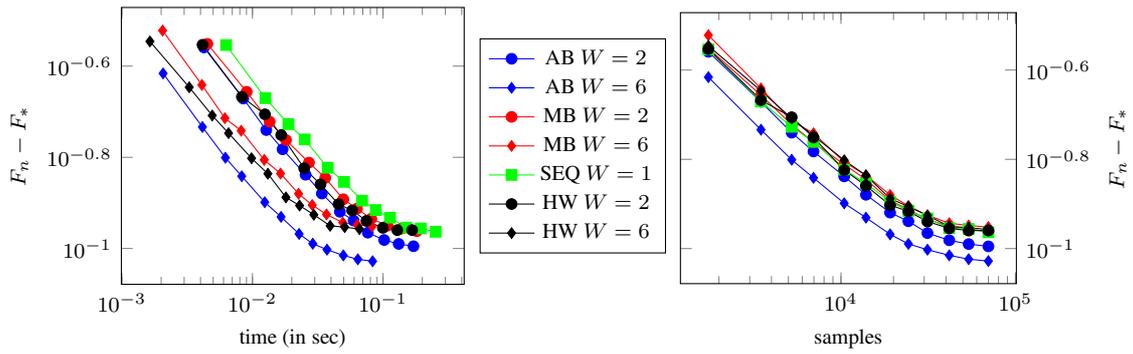
\begin{figure}
\centering
  \begin{subfigure}[t]{.58\linewidth}
    \centering
  \begin{tikzpicture}[
    ab/.style={color=blue, solid,mark options={fill=blue}},
    mb/.style={color=red, solid,mark options={fill=red}},
    seq/.style={color=green,solid,mark options={fill=green}},
    hw/.style={color=black, solid,mark options={fill=black}},
    every node/.style={font=\fontsize{8}{5}\selectfont}
  ]
    \begin{loglogaxis}[
        xlabel={time (in sec)},
        ylabel={$F_n - F_*$},
        legend entries={
            {{AB $W=2$},
             {AB $W=6$},
             {MB $W=2$},
             {MB $W=6$},
             {SEQ $W=1$},
             {HW $W=2$},
             {HW $W=6$},
             }
        },
        legend style={
            font=\fontsize{5}{5}\selectfont,at={(1.045,0.5)},anchor=west,
            legend columns=1},
        style={font=\footnotesize},
        width=0.7\textwidth,
        skip coords between index={0}{1}]
        \def\myroot{tables__parallel__news20}

        \pgfplotstableread{\myroot__convergence_depth_total_primal_test_group=ab=10,2.table}\tablea
        \addplot+[ab,mark=*] table[x=x, y=y, y error=y_error] {\tablea};

        \pgfplotstableread{\myroot__convergence_depth_total_primal_test_group=ab=10,6.table}\tablea
        \addplot+[ab, mark=diamond*] table[x=x, y=y, y error=y_error] {\tablea};

        \pgfplotstableread{\myroot__convergence_depth_total_primal_test_group=mb=10,2.table}\tablea
        \addplot+[mb,mark=*] table[x=x, y=y, y error=y_error] {\tablea};

        \pgfplotstableread{\myroot__convergence_depth_total_primal_test_group=mb=10,6.table}\tablea
        \addplot+[mb, mark=diamond*] table[x=x, y=y, y error=y_error] {\tablea};

                \pgfplotstableread{\myroot__convergence_depth_total_primal_test_group=hw=20,1.table}\tablea
        \addplot+[seq, mark=square*] table[x=x, y=y, y error=y_error] {\tablea};

        \pgfplotstableread{\myroot__convergence_depth_total_primal_test_group=hw=20,2.table}\tablea
        \addplot+[hw,mark=*] table[x=x, y=y, y error=y_error] {\tablea};

        \pgfplotstableread{\myroot__convergence_depth_total_primal_test_group=hw=20,6.table}\tablea
        \addplot+[hw,mark=diamond*] table[x=x, y=y, y error=y_error] {\tablea};
    \end{loglogaxis}
\end{tikzpicture}
\end{subfigure}%
\hfill
\begin{subfigure}[t]{0.40\textwidth}
  \centering
  \begin{tikzpicture}[
    ab/.style={color=blue, solid,mark options={fill=blue}},
    mb/.style={color=red, solid,mark options={fill=red}},
    seq/.style={color=green,solid,mark options={fill=green}},
    hw/.style={color=black, solid,mark options={fill=black}},
    every node/.style={font=\fontsize{8}{5}\selectfont}
  ]
    \begin{loglogaxis}[
        xlabel={samples},
        ylabel={$F_n - F_*$},
        yticklabel pos=right,
        style={font=\footnotesize},
        width=\textwidth,
        skip coords between index={0}{1}]
        \def\myroot{tables__parallel__news20}

        \pgfplotstableread{\myroot__convergence_total_primal_test_group=ab=10,2.table}\tablea
        \addplot+[ab,mark=*] table[x=x, y=y, y error=y_error] {\tablea};

        \pgfplotstableread{\myroot__convergence_total_primal_test_group=ab=10,6.table}\tablea
        \addplot+[ab, mark=diamond*] table[x=x, y=y, y error=y_error] {\tablea};

        \pgfplotstableread{\myroot__convergence_total_primal_test_group=mb=10,2.table}\tablea
        \addplot+[mb,mark=*] table[x=x, y=y, y error=y_error] {\tablea};

        \pgfplotstableread{\myroot__convergence_total_primal_test_group=mb=10,6.table}\tablea
        \addplot+[mb, mark=diamond*] table[x=x, y=y, y error=y_error] {\tablea};

                \pgfplotstableread{\myroot__convergence_total_primal_test_group=hw=20,1.table}\tablea
        \addplot+[seq, mark=square*] table[x=x, y=y, y error=y_error] {\tablea};

        \pgfplotstableread{\myroot__convergence_total_primal_test_group=hw=20,2.table}\tablea
        \addplot+[hw,mark=*] table[x=x, y=y, y error=y_error] {\tablea};

        \pgfplotstableread{\myroot__convergence_total_primal_test_group=hw=20,6.table}\tablea
        \addplot+[hw,mark=diamond*] table[x=x, y=y, y error=y_error] {\tablea};
    \end{loglogaxis}
\end{tikzpicture}
\end{subfigure}
\vspace*{-.1cm}
\caption{\label{an:news20:conv:nonstrict}Convergence result for \emph{news20}.
The error is given either as a function of the wall-clock time (left) or of the number
of samples processed (right).}
\end{figure}

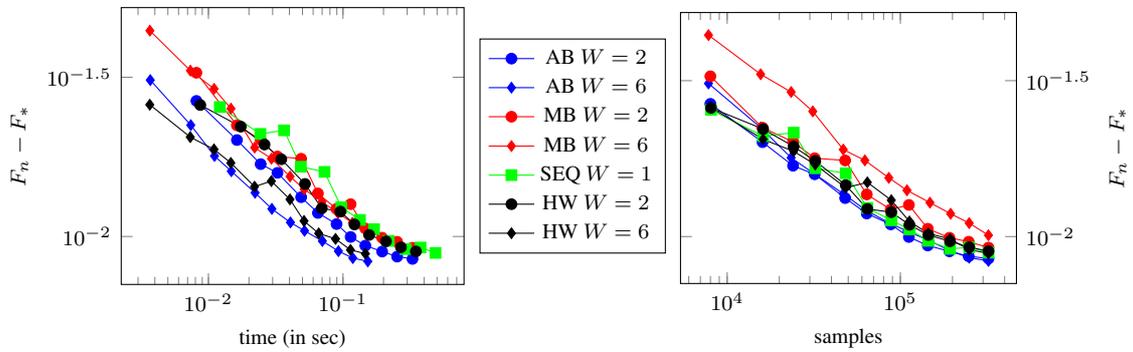
\begin{figure}
\centering
  \begin{subfigure}[t]{.58\linewidth}
    \centering
  \begin{tikzpicture}[
    ab/.style={color=blue, solid,mark options={fill=blue}},
    mb/.style={color=red, solid,mark options={fill=red}},
    seq/.style={color=green,solid,mark options={fill=green}},
    hw/.style={color=black, solid,mark options={fill=black}},
    every node/.style={font=\fontsize{8}{5}\selectfont}
  ]
    \begin{loglogaxis}[
        xlabel={time (in sec)},
        ylabel={$F_n - F_*$},
        legend entries={
            {{AB $W=2$},
             {AB $W=6$},
             {MB $W=2$},
             {MB $W=6$},
             {SEQ $W=1$},
             {HW $W=2$},
             {HW $W=6$},
             }
        },
        legend style={
            font=\fontsize{5}{5}\selectfont,at={(1.045,0.5)},anchor=west,
            legend columns=1},
        style={font=\footnotesize},
        width=0.7\textwidth,
        skip coords between index={0}{1}]
        \def\myroot{tables__parallel__spam}

        \pgfplotstableread{\myroot__convergence_depth_total_primal_test_group=ab=50,2.table}\tablea
        \addplot+[ab,mark=*] table[x=x, y=y, y error=y_error] {\tablea};

        \pgfplotstableread{\myroot__convergence_depth_total_primal_test_group=ab=50,6.table}\tablea
        \addplot+[ab, mark=diamond*] table[x=x, y=y, y error=y_error] {\tablea};

        \pgfplotstableread{\myroot__convergence_depth_total_primal_test_group=mb=50,2.table}\tablea
        \addplot+[mb,mark=*] table[x=x, y=y, y error=y_error] {\tablea};

        \pgfplotstableread{\myroot__convergence_depth_total_primal_test_group=mb=50,6.table}\tablea
        \addplot+[mb, mark=diamond*] table[x=x, y=y, y error=y_error] {\tablea};

                \pgfplotstableread{\myroot__convergence_depth_total_primal_test_group=hw=20,1.table}\tablea
        \addplot+[seq, mark=square*] table[x=x, y=y, y error=y_error] {\tablea};

        \pgfplotstableread{\myroot__convergence_depth_total_primal_test_group=hw=20,2.table}\tablea
        \addplot+[hw,mark=*] table[x=x, y=y, y error=y_error] {\tablea};

        \pgfplotstableread{\myroot__convergence_depth_total_primal_test_group=hw=20,6.table}\tablea
        \addplot+[hw,mark=diamond*] table[x=x, y=y, y error=y_error] {\tablea};
    \end{loglogaxis}
\end{tikzpicture}
\end{subfigure}%
\hfill
\begin{subfigure}[t]{0.40\textwidth}
  \centering
  \begin{tikzpicture}[
    ab/.style={color=blue, solid,mark options={fill=blue}},
    mb/.style={color=red, solid,mark options={fill=red}},
    seq/.style={color=green,solid,mark options={fill=green}},
    hw/.style={color=black, solid,mark options={fill=black}},
    every node/.style={font=\fontsize{8}{5}\selectfont}
  ]
    \begin{loglogaxis}[
        xlabel={samples},
        ylabel={$F_n - F_*$},
        yticklabel pos=right,
        style={font=\footnotesize},
        width=\textwidth,
        skip coords between index={0}{1}]
        \def\myroot{tables__parallel__spam}

        \pgfplotstableread{\myroot__convergence_total_primal_test_group=ab=50,2.table}\tablea
        \addplot+[ab,mark=*] table[x=x, y=y, y error=y_error] {\tablea};

        \pgfplotstableread{\myroot__convergence_total_primal_test_group=ab=50,6.table}\tablea
        \addplot+[ab, mark=diamond*] table[x=x, y=y, y error=y_error] {\tablea};

        \pgfplotstableread{\myroot__convergence_total_primal_test_group=mb=50,2.table}\tablea
        \addplot+[mb,mark=*] table[x=x, y=y, y error=y_error] {\tablea};

        \pgfplotstableread{\myroot__convergence_total_primal_test_group=mb=50,6.table}\tablea
        \addplot+[mb, mark=diamond*] table[x=x, y=y, y error=y_error] {\tablea};

                \pgfplotstableread{\myroot__convergence_total_primal_test_group=hw=20,1.table}\tablea
        \addplot+[seq, mark=square*] table[x=x, y=y, y error=y_error] {\tablea};

        \pgfplotstableread{\myroot__convergence_total_primal_test_group=hw=20,2.table}\tablea
        \addplot+[hw,mark=*] table[x=x, y=y, y error=y_error] {\tablea};

        \pgfplotstableread{\myroot__convergence_total_primal_test_group=hw=20,6.table}\tablea
        \addplot+[hw,mark=diamond*] table[x=x, y=y, y error=y_error] {\tablea};
    \end{loglogaxis}
\end{tikzpicture}
\end{subfigure}
\vspace*{-.1cm}
\caption{\label{an:spam:conv:nonstrict}Convergence result for \emph{spam}.
The error is given either as a function of the wall-clock time (left) or of the number
of samples processed (right).}
\end{figure}

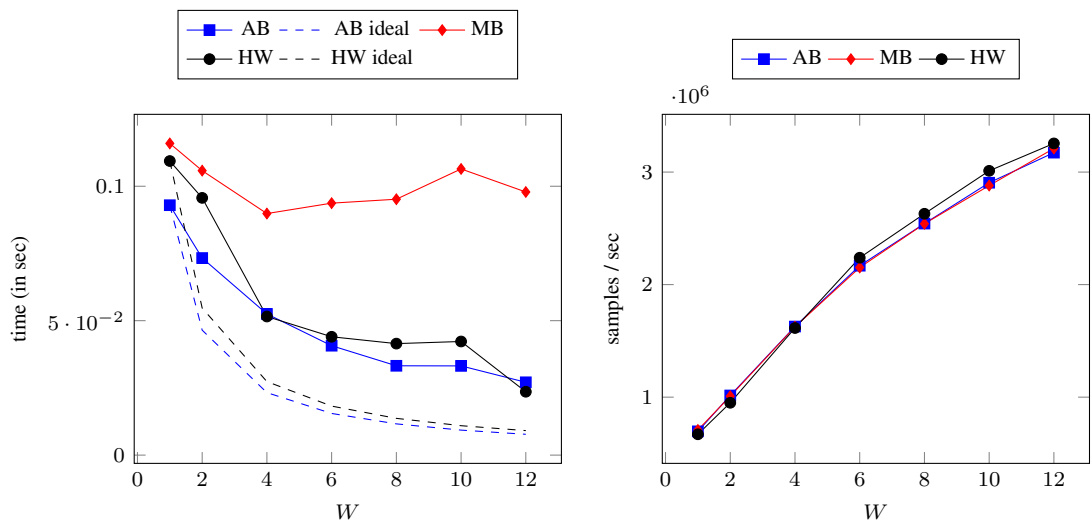
\begin{figure}

\vspace*{.2cm}

\begin{subfigure}[t]{0.48\textwidth}
\centering
\begin{tikzpicture}[
    ab/.style={color=blue, solid,mark options={fill=blue}},
    mb/.style={color=red, solid,mark options={fill=red}},
    seq/.style={color=green,solid,mark options={fill=green}},
    hw/.style={color=black, solid,mark options={fill=black}},
    every node/.style={font=\fontsize{8}{5}\selectfont}
  ]
    \begin{axis}[
        xlabel={$W$},
        ylabel={time (in sec)},
        legend entries={
            {{AB},
             {AB ideal},
             {MB},
             {HW},
             {HW ideal}
             }
        },
        legend style={
            font=\fontsize{7}{5}\selectfont,at={(0.5,1.1)},anchor=south,
            legend columns=3},
        style={font=\footnotesize},
        width=\textwidth]
        \def\myroot{tables__parallel__spam}

        \pgfplotstableread{\myroot__time_to_target_group=ab=50.table}\tablea
        \addplot+[ab,mark=square*] table[x=x, y=y, y error=y_error] {\tablea};
        \addplot[ab,dashed] table[x=x, y expr=0.0929490/\thisrow{x}, y error=y_error] {\tablea};

        \pgfplotstableread{\myroot__time_to_target_group=mb=50.table}\tablea
        \addplot+[mb, mark=diamond*] table[x=x, y=y, y error=y_error] {\tablea};

        \pgfplotstableread{\myroot__time_to_target_group=hw=20.table}\tablea
        \addplot+[hw,mark=*] table[x=x, y=y, y error=y_error] {\tablea};
        \addplot[hw,dashed] table[x=x, y expr=0.10937924/\thisrow{x}, y error=y_error] {\tablea};
    \end{axis}
\end{tikzpicture}
\end{subfigure}%
\hfill%
\begin{subfigure}[t]{0.48\textwidth}
\begin{tikzpicture}[
    ab/.style={color=blue, solid,mark options={fill=blue}},
    mb/.style={color=red, solid,mark options={fill=red}},
    seq/.style={color=green,solid,mark options={fill=green}},
    hw/.style={color=black, solid,mark options={fill=black}},
    every node/.style={font=\fontsize{8}{5}\selectfont}
  ]
    \begin{axis}[
        xlabel={$W$},
        ylabel={samples / sec},
        legend entries={
            {{AB},
             {MB},
             {HW},
             }
        },
        legend style={
            font=\fontsize{7}{5}\selectfont,at={(0.5,1.1)},anchor=south,
            legend columns=3},
        style={font=\footnotesize},
        width=\textwidth]
        \pgfplotstableread{tables__parallel__spam__speedup_group=ab=50.table}\tablea
        \addplot+[ab,mark=square*] table[x=x, y=y, y error=y_error] {\tablea};

        \pgfplotstableread{tables__parallel__spam__speedup_group=mb=50.table}\tablea
        \addplot+[mb, mark=diamond*] table[x=x, y=y, y error=y_error] {\tablea};

        \pgfplotstableread{tables__parallel__spam__speedup_group=hw=20.table}\tablea
        \addplot+[hw,mark=*] table[x=x, y=y, y error=y_error] {\tablea};
    \end{axis}
\end{tikzpicture}
\end{subfigure}

\vspace*{-.1cm}

\caption{\label{an:time_to_target_ips:spam} On the left, time to achieve a given test error when varying
the number of workers for the \emph{spam} dataset. The dashed line represents an ideal speedup dividing the time for 1 worker by $W$.
On the right, number of process sampled per second as a function of $W$ for the \emph{spam} dataset.}
\end{figure}

\begin{figure}
\centering
  \begin{subfigure}[t]{.58\linewidth}
    \centering
  \begin{tikzpicture}[
    ab/.style={color=blue, solid,mark options={fill=blue}},
    mb/.style={color=red, solid,mark options={fill=red}},
    seq/.style={color=green,solid,mark options={fill=green}},
    hw/.style={color=black, solid,mark options={fill=black}},
    every node/.style={font=\fontsize{8}{5}\selectfont}
  ]
    \begin{loglogaxis}[
        xlabel={time (in sec)},
        ylabel={$F_n - F_*$},
        legend entries={
            {{AB $W=2$},
             {AB $W=6$},
             {MB $W=2$},
             {MB $W=6$},
             {SEQ $W=1$},
             {HW $W=2$},
             {HW $W=6$},
             }
        },
        legend style={
            font=\fontsize{5}{5}\selectfont,at={(1.045,0.5)},anchor=west,
            legend columns=1},
        style={font=\footnotesize},
        width=0.7\textwidth,
        skip coords between index={0}{1}]
        \def\myroot{tables__parallel__url}

        \pgfplotstableread{\myroot__convergence_depth_total_primal_test_group=ab=50,2.table}\tablea
        \addplot+[ab,mark=*] table[x=x, y=y, y error=y_error] {\tablea};

        \pgfplotstableread{\myroot__convergence_depth_total_primal_test_group=ab=50,6.table}\tablea
        \addplot+[ab, mark=diamond*] table[x=x, y=y, y error=y_error] {\tablea};

        \pgfplotstableread{\myroot__convergence_depth_total_primal_test_group=mb=50,2.table}\tablea
        \addplot+[mb,mark=*] table[x=x, y=y, y error=y_error] {\tablea};

        \pgfplotstableread{\myroot__convergence_depth_total_primal_test_group=mb=50,6.table}\tablea
        \addplot+[mb, mark=diamond*] table[x=x, y=y, y error=y_error] {\tablea};

                \pgfplotstableread{\myroot__convergence_depth_total_primal_test_group=hw=20,1.table}\tablea
        \addplot+[seq, mark=square*] table[x=x, y=y, y error=y_error] {\tablea};

        \pgfplotstableread{\myroot__convergence_depth_total_primal_test_group=hw=20,2.table}\tablea
        \addplot+[hw,mark=*] table[x=x, y=y, y error=y_error] {\tablea};

        \pgfplotstableread{\myroot__convergence_depth_total_primal_test_group=hw=20,6.table}\tablea
        \addplot+[hw,mark=diamond*] table[x=x, y=y, y error=y_error] {\tablea};
    \end{loglogaxis}
\end{tikzpicture}
\end{subfigure}%
\hfill
\begin{subfigure}[t]{0.40\textwidth}
  \centering
  \begin{tikzpicture}[
    ab/.style={color=blue, solid,mark options={fill=blue}},
    mb/.style={color=red, solid,mark options={fill=red}},
    seq/.style={color=green,solid,mark options={fill=green}},
    hw/.style={color=black, solid,mark options={fill=black}},
    every node/.style={font=\fontsize{8}{5}\selectfont}
  ]
    \begin{loglogaxis}[
        xlabel={samples},
        ylabel={$F_n - F_*$},
        yticklabel pos=right,
        style={font=\footnotesize},
        width=\textwidth,
        skip coords between index={0}{1}]
        \def\myroot{tables__parallel__url}

        \pgfplotstableread{\myroot__convergence_total_primal_test_group=ab=50,2.table}\tablea
        \addplot+[ab,mark=*] table[x=x, y=y, y error=y_error] {\tablea};

        \pgfplotstableread{\myroot__convergence_total_primal_test_group=ab=50,6.table}\tablea
        \addplot+[ab, mark=diamond*] table[x=x, y=y, y error=y_error] {\tablea};

        \pgfplotstableread{\myroot__convergence_total_primal_test_group=mb=50,2.table}\tablea
        \addplot+[mb,mark=*] table[x=x, y=y, y error=y_error] {\tablea};

        \pgfplotstableread{\myroot__convergence_total_primal_test_group=mb=50,6.table}\tablea
        \addplot+[mb, mark=diamond*] table[x=x, y=y, y error=y_error] {\tablea};

                \pgfplotstableread{\myroot__convergence_total_primal_test_group=hw=20,1.table}\tablea
        \addplot+[seq, mark=square*] table[x=x, y=y, y error=y_error] {\tablea};

        \pgfplotstableread{\myroot__convergence_total_primal_test_group=hw=20,2.table}\tablea
        \addplot+[hw,mark=*] table[x=x, y=y, y error=y_error] {\tablea};

        \pgfplotstableread{\myroot__convergence_total_primal_test_group=hw=20,6.table}\tablea
        \addplot+[hw,mark=diamond*] table[x=x, y=y, y error=y_error] {\tablea};
    \end{loglogaxis}
\end{tikzpicture}
\end{subfigure}
\vspace*{-.1cm}
\caption{\label{an:url:conv:nonstrict}Convergence result for \emph{url}.
The error is given either as a function of the wall-clock time (left) or of the number
of samples processed (right).}
\end{figure}
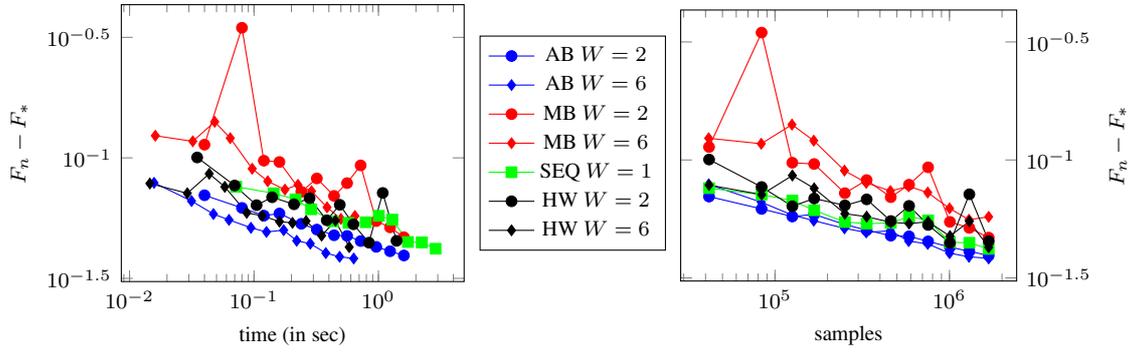

\begin{figure}
\vspace*{.2cm}

\begin{subfigure}[t]{0.48\textwidth}
\centering
\begin{tikzpicture}[
    ab/.style={color=blue, solid,mark options={fill=blue}},
    mb/.style={color=red, solid,mark options={fill=red}},
    seq/.style={color=green,solid,mark options={fill=green}},
    hw/.style={color=black, solid,mark options={fill=black}},
    every node/.style={font=\fontsize{8}{5}\selectfont}
  ]
    \begin{axis}[
        xlabel={$W$},
        ylabel={time (in sec)},
        legend entries={
            {{AB},
             {AB ideal},
             {MB},
             {HW},
             {HW ideal}
             }
        },
        legend style={
            font=\fontsize{7}{5}\selectfont,at={(0.5,1.1)},anchor=south,
            legend columns=3},
        style={font=\footnotesize},
        width=\textwidth]
        \def\myroot{tables__parallel__url}

        \pgfplotstableread{\myroot__time_to_target_group=ab=50.table}\tablea
        \addplot+[ab,mark=square*] table[x=x, y=y, y error=y_error] {\tablea};
        \addplot[ab,dashed] table[x=x, y expr=0.048273359/\thisrow{x}, y error=y_error] {\tablea};

        \pgfplotstableread{\myroot__time_to_target_group=mb=50.table}\tablea
        \addplot+[mb, mark=diamond*] table[x=x, y=y, y error=y_error] {\tablea};

        \pgfplotstableread{\myroot__time_to_target_group=hw=20.table}\tablea
        \addplot+[hw,mark=*] table[x=x, y=y, y error=y_error] {\tablea};
        \addplot[hw,dashed] table[x=x, y expr=0.0633868/\thisrow{x}, y error=y_error] {\tablea};
    \end{axis}
\end{tikzpicture}
\end{subfigure}%
\hfill%
\begin{subfigure}[t]{0.48\textwidth}
\begin{tikzpicture}[
    ab/.style={color=blue, solid,mark options={fill=blue}},
    mb/.style={color=red, solid,mark options={fill=red}},
    seq/.style={color=green,solid,mark options={fill=green}},
    hw/.style={color=black, solid,mark options={fill=black}},
    every node/.style={font=\fontsize{8}{5}\selectfont}
  ]
    \begin{axis}[
        xlabel={$W$},
        ylabel={samples / sec},
        legend entries={
            {{AB},
             {MB},
             {HW},
             }
        },
        legend style={
            font=\fontsize{7}{5}\selectfont,at={(0.5,1.1)},anchor=south,
            legend columns=3},
        style={font=\footnotesize},
        width=\textwidth]
        \pgfplotstableread{tables__parallel__url__speedup_group=ab=50.table}\tablea
        \addplot+[ab,mark=square*] table[x=x, y=y, y error=y_error] {\tablea};

        \pgfplotstableread{tables__parallel__url__speedup_group=mb=50.table}\tablea
        \addplot+[mb, mark=diamond*] table[x=x, y=y, y error=y_error] {\tablea};

        \pgfplotstableread{tables__parallel__url__speedup_group=hw=20.table}\tablea
        \addplot+[hw,mark=*] table[x=x, y=y, y error=y_error] {\tablea};
    \end{axis}
\end{tikzpicture}
\end{subfigure}

\vspace*{-.1cm}

\caption{\label{an:time_to_target_ips} On the left, time to achieve a given test error when varying
the number of workers on \emph{url}. The dashed line represents an ideal speedup dividing the time for 1 worker by $W$.
On the right, number of process sampled per second as a function of $W$ on \emph{url}.}
\end{figure}
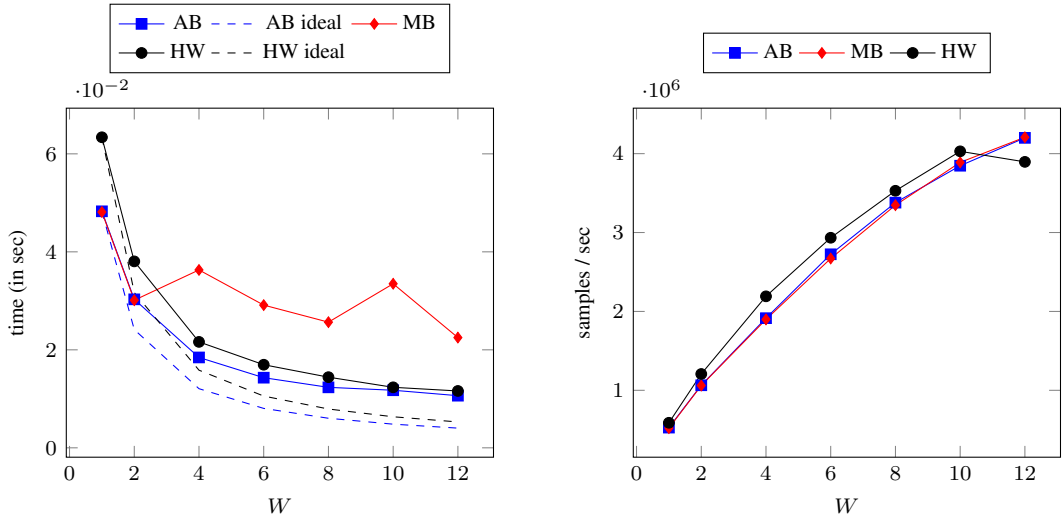

\clearpage

\ansubsection{Experimental results for SVRG}

We give a comparison of regular mini-batch and AdaBatch on news20 and spam
on figure~\ref{an:figure:svrg}. The difference is less marked than on
the url dataset which we believe is due to the relative simplicity of the optimization problem
on such datasets. The L2 regularization was chosen in order to achieve
relatively good validation error while retaining the good convergence of SVRG in the
strongly convex case.

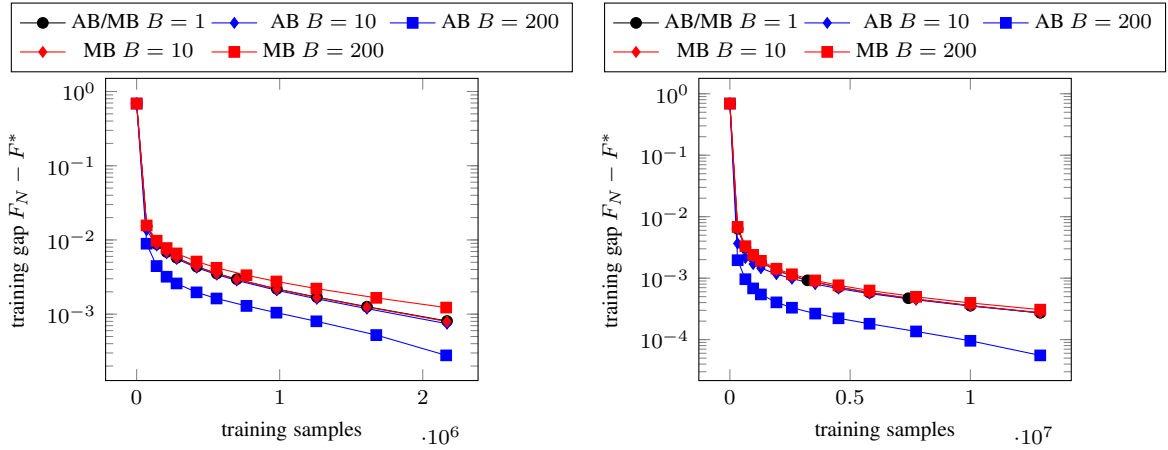
\begin{figure}[h]
\begin{subfigure}[t]{0.48\textwidth}
\centering
\begin{tikzpicture}[
    ab/.style={color=blue, solid,mark options={fill=blue}},
    mb/.style={color=red, solid,mark options={fill=red}},
    seq/.style={color=green,solid,mark options={fill=green}},
    hw/.style={color=black, solid,mark options={fill=black}},
    every node/.style={font=\fontsize{8}{5}\selectfont}
  ]
    \begin{semilogyaxis}[
        xlabel={training samples},
        ylabel={training gap $F_N - F^*$},
        name=plot1,
        legend entries={
            {
            {AB/MB $B=1$},
            {AB $B=10$},
            {AB $B=200$},
            {MB $B=10$},
            {MB $B=200$},
             }
        },
        legend style={
            font=\fontsize{7}{5}\selectfont,at={(0.5,1.02)},anchor=south,
            legend columns=3},
        style={font=\footnotesize},
        width=0.9\linewidth,
        skip coords between index={12}{14}]]
        \newcommand{\myplot}[1]{%
        \pgfplotstableread{tables__svrg__news20__convergence_total_gap_group=#1.table}\tablea}

        \myplot{ab=1}
        \addplot+[hw,mark=*] table[x=x, y=y, y error=y_error] {\tablea};
        \myplot{ab=10}
        \addplot+[ab,mark=diamond*] table[x=x, y=y, y error=y_error] {\tablea};
        \myplot{ab=200}
        \addplot+[ab,mark=square*] table[x=x, y=y, y error=y_error] {\tablea};

        \myplot{mb=10}
        \addplot+[mb,mark=diamond*] table[x=x, y=y, y error=y_error] {\tablea};
        \myplot{mb=200}
        \addplot+[mb,mark=square*] table[x=x, y=y, y error=y_error] {\tablea};
    \end{semilogyaxis}
\end{tikzpicture}
\vspace*{-.1cm}
\caption{\label{an:figure:svrg_news} Comparison of the training gap $F_n - F_*$ for regular mini-batch vs AdaBatch
with SVRG on \emph{news20} for the log loss with an L2 penalty of $\frac{10^{-5}}{2} \norm{w}_{\diag{p}}^2$.}
\end{subfigure}%
\hfill%
\begin{subfigure}[t]{0.48\textwidth}
\centering
\begin{tikzpicture}[
    ab/.style={color=blue, solid,mark options={fill=blue}},
    mb/.style={color=red, solid,mark options={fill=red}},
    seq/.style={color=green,solid,mark options={fill=green}},
    hw/.style={color=black, solid,mark options={fill=black}},
    every node/.style={font=\fontsize{8}{5}\selectfont}
  ]
    \begin{semilogyaxis}[
        xlabel={training samples},
        ylabel={training gap $F_N - F^*$},
        name=plot1,
        legend entries={
            {
            {AB/MB $B=1$},
            {AB $B=10$},
            {AB $B=200$},
            {MB $B=10$},
            {MB $B=200$},
             }
        },
        legend style={
            font=\fontsize{7}{5}\selectfont,at={(0.5,1.02)},anchor=south,
            legend columns=3},
        style={font=\footnotesize},
        width=0.9\linewidth]
        \newcommand{\myplot}[1]{%
        \pgfplotstableread{tables__svrg__spam__convergence_total_gap_group=#1.table}\tablea}

        \myplot{ab=1}
        \addplot+[hw,mark=*] table[x=x, y=y, y error=y_error] {\tablea};
        \myplot{ab=10}
        \addplot+[ab,mark=diamond*] table[x=x, y=y, y error=y_error] {\tablea};
        \myplot{ab=200}
        \addplot+[ab,mark=square*] table[x=x, y=y, y error=y_error] {\tablea};

        \myplot{mb=10}
        \addplot+[mb,mark=diamond*] table[x=x, y=y, y error=y_error] {\tablea};
        \myplot{mb=200}
        \addplot+[mb,mark=square*] table[x=x, y=y, y error=y_error] {\tablea};
    \end{semilogyaxis}
\end{tikzpicture}
\vspace*{-.1cm}
\caption{\label{an:figure:svrg_spam} Comparison of the training gap $F_n - F_*$ for regular mini-batch vs AdaBatch
with SVRG on \emph{spam} for the log loss with an L2 penalty of $\frac{10^{-6}}{2} \norm{w}_{\diag{p}}^2$.}
\end{subfigure}
\caption{\label{an:figure:svrg}Comparison of regular mini-batch vs AdaBatch with SVRG on \emph{news20} and \emph{spam} dataset.}
\end{figure}

\clearpage
\bibliography{biblio}

\begin{thebibliography}{10}

\bibitem{NIPS2007_3323}
Olivier Bousquet and L\'{e}on Bottou.
\newblock The tradeoffs of large scale learning.
\newblock In {\em Advances in Neural Information Processing Systems 21}, 2008.

\bibitem{bach11}
Francis Bach and Eric Moulines.
\newblock {Non-Asymptotic Analysis of Stochastic Approximation Algorithms for
  Machine Learning}.
\newblock In {\em Advances in Neural Information Processing Systems 24}, 2011.

\bibitem{ads_from_trenches}
H.~Brendan McMahan, Gary Holt, D.~Sculley, Michael Young, Dietmar Ebner, Julian
  Grady, Lan Nie, Todd Phillips, Eugene Davydov, Daniel Golovin, Sharat
  Chikkerur, Dan Liu, Martin Wattenberg, Arnar~Mar Hrafnkelsson, Tom Boulos,
  and Jeremy Kubica.
\newblock Ad click prediction: a view from the trenches.
\newblock In {\em Proceedings of the ACM SIGKDD International Conference on
  Knowledge Discovery and Data Mining (KDD)}, 2013.

\bibitem{multicore}
Balaji Venu.
\newblock Multi-core processors - an overview.
\newblock Technical Report 1110.3535, arXiv, 2011.

\bibitem{parallel_sgd_average}
Martin Zinkevich, Markus Weimer, Lihong Li, and Alex~J. Smola.
\newblock Parallelized stochastic gradient descent.
\newblock In {\em Advances in Neural Information Processing Systems 23}, 2010.

\bibitem{slow_learners_are_fast}
Martin Zinkevich, John Langford, and Alex~J. Smola.
\newblock Slow learners are fast.
\newblock In {\em Advances in Neural Information Processing Systems 22}, 2009.

\bibitem{Hogwild}
Feng Niu, Benjamin Recht, Christopher Re, and Stephen~J. Wright.
\newblock {HOGWILD!: A Lock-Free Approach to Parallelizing Stochastic Gradient
  Descent}.
\newblock In {\em Advances in Neural Information Processing Systems 24}, 2011.

\bibitem{passcode}
Cho-Jui Hsieh, Hsiang-Fu Yu, and Inderjit Dhillon.
\newblock Passcode: Parallel asynchronous stochastic dual co-ordinate descent.
\newblock In {\em International Conference on Machine Learning}, 2015.

\bibitem{hogwild_mania}
H.~{Mania}, X.~{Pan}, D.~{Papailiopoulos}, B.~{Recht}, K.~{Ramchandran}, and
  M.~I. {Jordan}.
\newblock {Perturbed Iterate Analysis for Asynchronous Stochastic
  Optimization}.
\newblock Technical Report 1507.06970, arXiv, 2015.

\bibitem{online_minibatch}
Ofer Dekel, Ran Gilad-Bachrach, Ohad Shamir, and Lin Xiao.
\newblock Optimal distributed online prediction using mini-batches.
\newblock {\em Journal of Machine Learning Research}, 2012.

\bibitem{minibatch}
P.~{Jain}, S.~M. {Kakade}, R.~{Kidambi}, P.~{Netrapalli}, and A.~{Sidford}.
\newblock {Parallelizing Stochastic Approximation Through Mini-Batching and
  Tail-Averaging}.
\newblock Technical Report 1610.03774, arXiv, 2016.

\bibitem{adagrad}
John Duchi, Elad Hazan, and Yoram Singer.
\newblock Adaptive subgradient methods for online learning and stochastic
  optimization.
\newblock {\em Journal of Machine Learning Research}, 2011.

\bibitem{tonga}
Nicolas~Le Roux, Pierre-Antoine Manzagol, and Yoshua Bengio.
\newblock Topmoumoute online natural gradient algorithm.
\newblock In {\em Advances in Neural Information Processing Systems 20}, 2008.

\bibitem{minibatch_sgd}
Mu~Li, Tong Zhang, Yuqiang Chen, and Alexander~J. Smola.
\newblock Efficient mini-batch training for stochastic optimization.
\newblock In {\em Proceedings of the ACM SIGKDD International Conference on
  Knowledge Discovery and Data Mining}, 2014.

\bibitem{asaga}
R.~{Leblond}, F.~{Pedregosa}, and S.~{Lacoste-Julien}.
\newblock {ASAGA: Asynchronous Parallel SAGA}.
\newblock In {\em Proceedings of the 20th International Conference on
  Artificial Intelligence and Statistics}, 2017.

\bibitem{cyclades}
X.~{Pan}, M.~{Lam}, S.~{Tu}, D.~{Papailiopoulos}, C.~{Zhang}, M.~I. {Jordan},
  K.~{Ramchandran}, C.~{Re}, and B.~{Recht}.
\newblock {CYCLADES: Conflict-free Asynchronous Machine Learning}.
\newblock In {\em Advances in Neural Information Processing Systems 29}, 2016.

\bibitem{duchi2013estimation}
John Duchi, Michael~I Jordan, and Brendan McMahan.
\newblock Estimation, optimization, and parallelism when data is sparse.
\newblock In {\em Advances in Neural Information Processing Systems 26}, 2013.

\bibitem{bach13}
Francis Bach and Eric Moulines.
\newblock Non-strongly-convex smooth stochastic approximation with convergence
  rate \textsc{O}(1/n).
\newblock In {\em Advances in Neural Information Processing Systems 26}, 2013.

\bibitem{sutskever2013importance}
Ilya Sutskever, James Martens, George Dahl, and Geoffrey Hinton.
\newblock On the importance of initialization and momentum in deep learning.
\newblock In {\em International Conference on Machine Learning}, 2013.

\bibitem{cocoa}
Chenxin Ma, Virginia Smith, Martin Jaggi, Michael~I. Jordan, Peter
  Richt{\'{a}}rik, and Martin Tak{\'{a}}c.
\newblock Adding vs. averaging in distributed primal-dual optimization.
\newblock In {\em International Conference on Machine Learning}, 2015.

\bibitem{proof_sgd}
D.~{Needell} and R.~{Ward}.
\newblock {Batched Stochastic Gradient Descent with Weighted Sampling}.
\newblock Technical Report 1608.07641, arXiv, 2016.

\bibitem{nesterov2008confidence}
Yu~Nesterov and J-Ph Vial.
\newblock Confidence level solutions for stochastic programming.
\newblock {\em Automatica}, 2008.

\bibitem{defossez}
Alexandre D{\'e}fossez and Francis Bach.
\newblock Averaged least-mean-squares: Bias-variance trade-offs and optimal
  sampling distributions.
\newblock In {\em Artificial Intelligence and Statistics}, 2015.

\bibitem{svrg}
Rie Johnson and Tong Zhang.
\newblock Accelerating stochastic gradient descent using predictive variance
  reduction.
\newblock In {\em Advances in Neural Information Processing Systems 26}, 2013.

\bibitem{defazio2014saga}
Aaron Defazio, Francis Bach, and Simon Lacoste-Julien.
\newblock {SAGA}: A fast incremental gradient method with support for
  non-strongly convex composite objectives.
\newblock In {\em Advances in Neural Information Processing Systems 27}, 2014.

\bibitem{url}
Justin Ma, Lawrence~K. Saul, Stefan Savage, and Geoffrey~M. Voelker.
\newblock Identifying suspicious urls: An application of large-scale online
  learning.
\newblock In {\em International Conference on Machine Learning}, 2009.

\bibitem{bubeck2015convex}
S{\'e}bastien Bubeck.
\newblock Convex optimization: Algorithms and complexity.
\newblock {\em Foundations and Trends{\textregistered} in Machine Learning},
  2015.

\end{thebibliography}
\end{document}